\pdfoutput=1
\PassOptionsToPackage{table,xcdraw}{xcolor}

\documentclass[11pt]{article}

\usepackage[final]{acl}

\usepackage{times}
\usepackage{latexsym}

\usepackage[T1]{fontenc}

\usepackage[utf8]{inputenc}

\usepackage{microtype}

\usepackage{inconsolata}

\usepackage{graphicx}

\usepackage[T1]{fontenc}    
\usepackage{hyperref}       
\usepackage{url}            
\usepackage{booktabs}       
\usepackage{amsfonts}       
\usepackage{nicefrac}       
\usepackage{xcolor}         

\usepackage{algorithm}
\usepackage{algpseudocode}
\usepackage{amsmath}
\usepackage{amsthm}
\usepackage{mathtools}
\usepackage{bbm}
\usepackage{amssymb}
\usepackage{multirow}
\usepackage{makecell}  
\usepackage{ragged2e}   

\newcolumntype{L}[1]{>{\arraybackslash}p{#1}}

\newcolumntype{C}[1]{>{\centering\arraybackslash}p{#1}}

\definecolor{mygreen}{HTML}{007F00}
\newcommand{\greencheck}{\textcolor{mygreen}{$\checkmark$}}
\newcommand{\redx}{\textcolor{red}{$\times$}}

\usepackage[capitalize,noabbrev]{cleveref}

\newcommand\prob[1]{\underset{#1}{\Pr}\,}

\usepackage{tikz}
\usetikzlibrary{arrows.meta, positioning}

\definecolor{myblue}{HTML}{0088ff}  
\definecolor{mygreen}{HTML}{55ee11}  
\definecolor{mydarkgreen}{HTML}{007F00} 
\usepackage{bbm}

\newcommand{\OurProb}{Autoregressive Reasoning Entailment Stability}
\newcommand{\OURS}{ARES}
\newcommand{\cert}{\OURS{}}
\newcommand{\entailraw}{Entail-Base}
\newcommand{\entailprev}{Entail-Prev}
\newcommand{\llmjudge}{LLM-Judge}
\newcommand{\prm}{PRM}
\newcommand{\ourdata}{ClaimTrees}
\newcommand{\recipes}{CaptainCookRecipes}

\usepackage{tcolorbox}
\usepackage{graphicx}
\usepackage{listings}

\usepackage{xurl}               
\usepackage[breaklinks]{hyperref}  

\usepackage{float} 
\usepackage{pifont}
\newcommand{\cmark}{\ding{51}}%
\newcommand{\xmark}{\ding{55}}%

\usepackage{tabularx}

\usepackage{thm-restate}
\theoremstyle{plain}
\newtheorem{theorem}{Theorem}[section]

\theoremstyle{definition}
\newtheorem{definition}[theorem]{Definition}

\theoremstyle{remark}

\newcommand{\entail}{\mathcal{E}}

\DeclarePairedDelimiter\abs{\lvert}{\rvert}

\newcommand\mcal[1]{\mathcal{#1}}
\newcommand\mrm[1]{\mathrm{#1}}
\newcommand\msf[1]{\mathsf{#1}}

\usepackage{tcolorbox}

\newcommand\penn{$^\dagger$}
\newcommand\utaustin{$^\star$}

\title{Probabilistic Soundness Guarantees in LLM Reasoning Chains}

\author{ \\
\textbf{Weiqiu You}\penn{} \quad \textbf{Anton Xue}\utaustin{} \quad \textbf{Shreya Havaldar}\penn{} \quad \textbf{Delip Rao}\penn{}  \quad \textbf{Helen Jin}\penn{} \\[2ex] \textbf{Chris Callison-Burch}\penn{} \quad \textbf{Eric Wong}\penn{} \\ \\
   \penn{}University of Pennsylvania \\
   \utaustin{}University of Texas at Austin \\
   }

\begin{document}
\maketitle

\begin{abstract}
In reasoning chains generated by large language models (LLMs), initial errors often propagate and undermine the reliability of the final conclusion.
Current LLM-based error detection methods often fail to detect propagated errors because earlier errors can corrupt judgments of downstream reasoning.
To better detect such errors, we introduce Autoregressive Reasoning Entailment Stability (ARES), a probabilistic framework that evaluates each reasoning step based solely on previously-verified premises.
This inductive method yields a nuanced score for each step and provides certified statistical guarantees of its soundness, rather than a brittle binary label. ARES achieves state-of-the-art performance across four benchmarks (72.1\% Macro-F1, +8.2 points) and demonstrates superior robustness on very long synthetic reasoning chains, where it excels at detecting propagated errors (90.3\% F1, +27.6 points).
\footnote{Correspondence to \texttt{weiqiuy@seas.upenn.edu}. Code is available at \url{https://github.com/fallcat/ares}.}

%

\end{abstract}

\section{Introduction}
\label{sec:introduction}

Large Language Models (LLMs) are taking on increasingly sophisticated reasoning tasks in critical fields like medicine and scientific discovery. Yet, a fundamental challenge persists: the chain-of-thought processes that lead to their outputs are frequently flawed with errors~\citep{huang2025survey,lyu-etal-2024-towards}. This critically compromises the reliability of LLM-generated content, diminishing user confidence and impeding the broader adoption of LLMs in high-stakes applications~\citep{agarwal2024faithfulness,chen2023quantifying}.


As illustrated in \cref{fig:unsound}, one type of error is an \textit{ungrounded step}---a step that is incorrect with respect to the given context.
For example, the model might incorrectly copy a 2/5 in the context to be 3/5.
Another common error is an \textit{invalid derivation}---for example, deriving $5x=9x-20$ from $x/(3x-7)=3/5$---which is a logical misstep or miscalculation~\citep{lee2025evaluatingstepbystepreasoningtraces}.
A third type of error involves \textit{error propagation}: even if the logic is valid, an incorrect starting assumption can lead to a wrong conclusion. For instance, using the incorrect claim $5x=9x-20$ to derive $x=5$ is logically valid but the derived claim is incorrect due to the initial error~\citep{tyagi-etal-2024-step}.

\begin{figure}
\centering

\includegraphics[width=\linewidth]{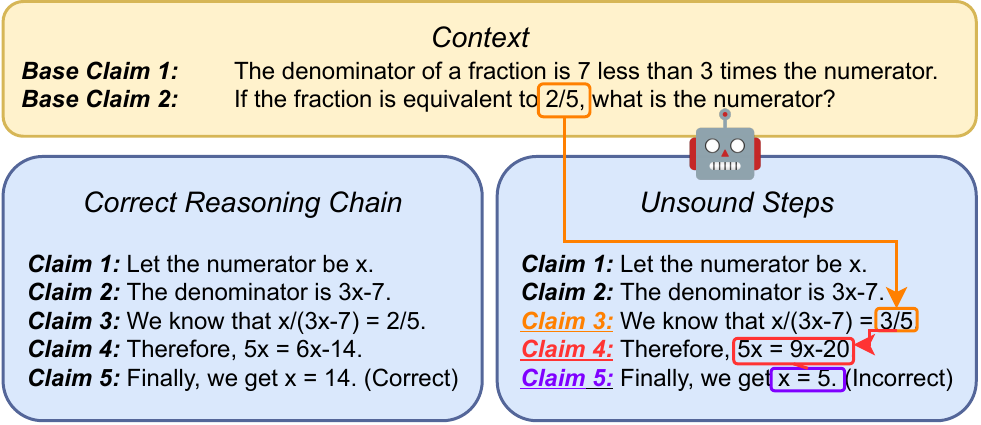}

\caption{
\textbf{Faulty LLM reasoning due to propagated errors from ungrounded and invalid steps.}
An unsound step is a step that is either {\color{orange}\textbf{ungrounded}} (incorrect with respect to the context), {\color{red}\textbf{invalid}} (logically incorrectly derived), or contains {\color{violet}\textbf{propagated}} errors.
In this example, {\color{orange}\textbf{Step 3 is ungrounded}} because it contains information different from the base claim 2.
{\color{red}\textbf{Step 4 is invalid}} because it contains an incorrect mathematical computation.
{\color{violet}\textbf{Step 5 is a propagated error}}, even though it is logically correct with respect to Step 4.
This figure is adapted from an example from \citet{lee2025evaluatingstepbystepreasoningtraces}.
}
\label{fig:unsound}
\end{figure}



\begin{figure*}[t]
    \centering
    \includegraphics[width=\linewidth]{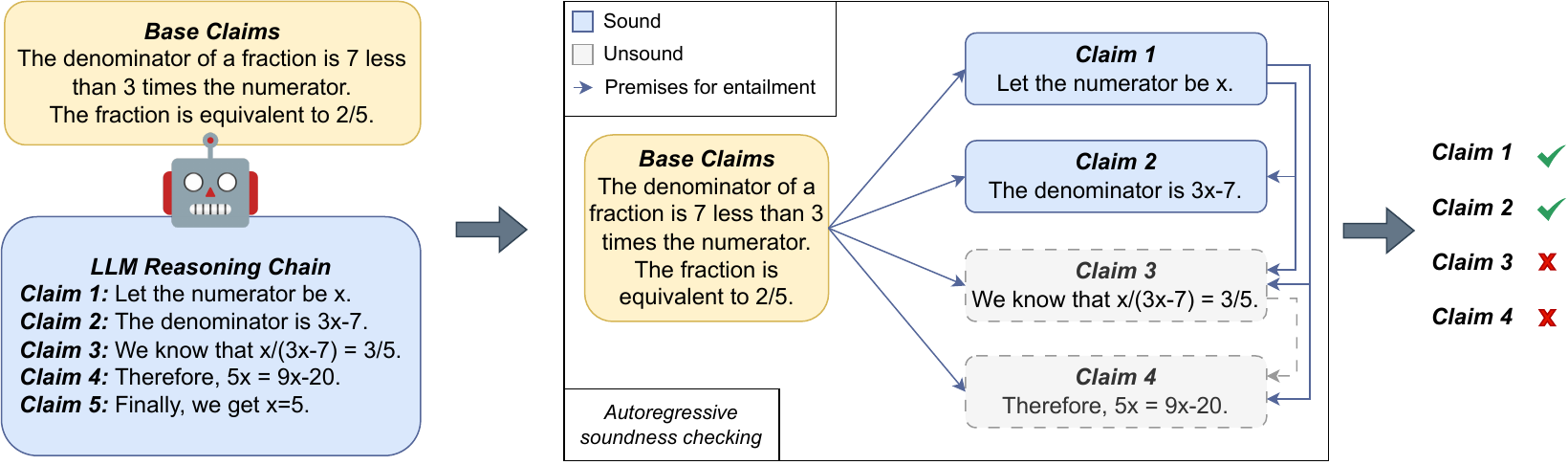}
    \caption{\textbf{(Autoregressive Soundness Checking)} When we verify an LLM generated reasoning chain, we can break the context and reasoning chain down to base claims and derived claims. An autoregressive soundness checker can then check each derived claim step-by-step, using only claims already identified to be sound as the premise.
    }
    \label{fig:step_by_step_process}
\end{figure*}


Current error detection methods typically aim to identify all errors at once.
For example, LLM judges are prompted to evaluate the entire chain and assess each step for correctness~\citep{tyagi-etal-2024-step,he2025largelanguagemodelsdetect}. Similarly, Process Reward Models (PRMs) are language models trained with step-level classification heads on this same objective~\citep{lightman2023letsverifystepstep}.

However, existing error detection methods often fall short.
Specifically, they are often distracted by the presence of propagated errors~\citep{he2025largelanguagemodelsdetect,turpin2023language,dhuliawala2023chain}. In the example from Figure~\ref{fig:unsound}, if steps 3, 4, and 5 are evaluated together, an LLM may incorrectly mark step 5 as sound by incorrectly relying on step 4, which is invalid.
This highlights the need for robust methods that can assess the soundness of each step without being adversely distracted by prior errors.

To address this issue, we draw inspiration from human reasoning.
Humans typically review claims sequentially, and disregard previously unsound statements when evaluating subsequent ones~\citep{johnson2010deductive,mukherjee2025premiseaugmentedreasoningchainsimprove}.
In contrast, LLMs struggle to ignore prior errors, which causes naive detection methods to fail at simultaneously identifying and localizing all errors in a reasoning chain~\citep{wu2024easily,song2024much}.
To overcome this limitation, we develop \OurProb{} (\OURS{}), a probabilistic framework that evaluates the soundness of each reasoning step based on its expected entailment probability, conditioned only on previously-occurring, sound claims (Figure~\ref{fig:step_by_step_process}).
We iteratively evaluate each claim as follows: \textit{entailed} claims are retained as premises for subsequent steps, while \textit{non-entailed} claims are discarded.
For \textit{uncertain} claims, retention is probabilistic based on the entailment model.
This adaptation not only improves error detection but also enables us to give certified guarantees on the robustness of a reasoning chain.

Our contributions are highlighted as follows.

\begin{itemize}
    \item We introduce \OurProb{} (\OURS{}), a novel probabilistic framework for evaluating claims in LLM reasoning chains. This framework uniquely assesses each step by conditioning only on previously verified sound claims, ensuring a robust and adaptable evaluation.

    \item We design a computationally and sample-efficient autoregressive algorithm for entailment estimation within this framework. Crucially, this algorithm provides sample-efficient certifications of entailment with rigorous statistical guarantees, a capability absent in prior methods.

    \item We demonstrate that \OURS{} accurately certifies both sound and unsound reasoning steps, particularly excelling in long chains prone to error propagation. \OURS{} significantly surpasses existing approaches and generalizes across diverse reasoning tasks.
\end{itemize}

\section{Soundness in Reasoning Chains}
\label{sec:background}

We aim to identify and certify errors within LLM-generated chain-of-thought (CoT) reasoning. To this end, this section formalizes reasoning chains in terms of their constituent claims (\cref{sec:claims_and_sequences}), introduces the concept of probabilistic entailment between these claims (\cref{sec:entailment}), and defines a notion of soundness that incorporates internal groundedness, validity, and the entailment of a final hypothesis (\cref{sec:soundness}).

\subsection{Claims and Sequences of Claims}
\label{sec:claims_and_sequences}
A reasoning chain is conceptualized as a sequence of claims, where a claim is the assertion of a proposition.
For instance, ``The denominator is $3x-7$'' is a claim regarding a component of an algebraic expression, while ``We know that $\tfrac{x}{3x-7}=\tfrac{2}{5}$'' is a claim that synthesizes prior information about an equation.
The granularity of claims is domain-dependent; it is permissible for a claim to range from an atomic statement or a single sentence (e.g., ``We can simplify $\tfrac{x}{3x-7}=\tfrac{2}{5}$ to $5x=6x-14$.'') to more extensive segments like entire theorems or proofs.

For a more formal discussion of our method, we let \(\mcal{C}\) denote the set of all possible claims, and \(\mcal{C}^\star\) represent the set of all possible sequences of claims. An example of such a sequence is as follows:
\begin{align*}
    \big(
    &\text{``Let the numerator be \(x\)''}, \\
    &\text{``The denominator is \(3x - 7\)''}, \\
    &\text{``We know that \(\tfrac{x}{3x-7} = \tfrac{2}{5}\)''}
    \big) \in \mcal{C}^\star
\end{align*}
which consists of the following individual claims:
\begin{align*}
    \text{``Let the numerator be \(x\)''} &\in \mcal{C}, \\
    \text{``The denominator is \(3x - 7\)''} &\in \mcal{C}, \\
    \text{``We know that \(\tfrac{x}{3x-7} = \tfrac{2}{5}\)''} &\in \mcal{C}.
\end{align*}
This distinction between individual claims and sequences of claims is important for discussing the inclusion and exclusion of items from a premise during logical entailment, which we define next.

\subsection{Probabilistic Entailment of Claims}
\label{sec:entailment}
To capture the notion of logical entailment between claims expressed in natural language, we introduce probabilistic entailment models. This approach is motivated by the inherent fuzziness and ambiguity often present in natural language reasoning~\citep{zadeh2008fuzzy,yu2024natural}.
Formally, a probabilistic entailment model $\entail{}:\mathcal{C}^\star \times \mathcal{C} \rightarrow [0, 1]$ accepts a sequence of claims as a premise, \(P \in \mcal{C}^\star\), and a single claim as a hypothesis, \(H \in \mcal{C}\). It then returns a scalar value representing the probability that the premise \(P\) entails the hypothesis \(H\).
For instance, consider the premise and hypothesis pair
\begin{align*}
    P &= \big(\text{``Sarah put on her running shoes.''}, \\
    &\qquad \text{``She stretched by the sidewalk.''}, \\
    &\qquad \text{``The sun was setting.''}\big) \\
    H &= \text{``Sarah is going for an evening run.''}
\end{align*}
A probabilistic entailment model might output \(\entail{}(P, H) = 0.85\).
This score reflects the linguistic and social ambiguity in inferring the certainty of an ``evening run'' from the actions of ``donning running shoes and stretching''. Such a fuzzy, probabilistic approach generalizes classical Boolean logic, where the output is strictly 1 for entailment and 0 for non-entailment.~\footnote{
We distinguish between a non-entailed claim (not logically following premises) and a provably false claim (factually incorrect). For instance, ``Sarah lives in Philadelphia'' is not entailed but not demonstrably false. 
}

\subsection{Reasoning Chains and Soundness}
\label{sec:soundness}

To analyze the step-by-step reasoning of LLMs, particularly in CoT processes, we conceptualize the output as a \textit{reasoning chain}. This chain initiates with a set of provided statements or contextual information, designated as \textit{base claims}. Following these, the LLM autoregressively produces a sequence of subsequent statements, which we term \textit{derived claims}. This entire sequence is formally represented as:
\begin{equation}
    (C_1, \ldots, C_n, C_{n+1}, \ldots, C_{n+m}) \in \mcal{C}^\star
\end{equation}
where \(C_1, \ldots, C_n\) are the base claims, and \(C_{n+1}, \ldots, C_{n+m}\) are the derived claims.

This partition is methodologically crucial. Base claims ($C_1, \ldots, C_n$) serve as the foundational premises for a given reasoning task; their factual accuracy is given and assumed to be validated by external mechanisms. Instead, we focus on assessing whether each derived claim ($C_{n+i}$ for $i=1, \ldots, m$) is soundly inferred from the set of preceding statements.
To begin, we define a deterministic (i.e., ``hard'') version of soundness, where all derived claims are entailed with certainty.


\begin{definition}[Hard Soundness]
\label{def:hard_soundness} 
A reasoning chain \((C_1, \ldots, C_{n+m})\) is \textit{hard-sound} with respect to the entailment model \(\entail{}\) if for all \(m\) derived claims, we have
\begin{equation}
\begin{gathered}
    \mcal{E}((C_1, \ldots, C_n), C_{n+1}) = 1 \\
    \vdots  \\
    \mcal{E}((C_1, \ldots, C_{n+m-1}), C_{n+m}) = 1 
\end{gathered}
\end{equation}
\end{definition}

The concept of hard soundness provides a precise, albeit strict, benchmark for evaluating the logical integrity of a reasoning chain: it requires every derived claim to be perfectly entailed by its predecessors.
However, LLM-generated reasoning chains often deviate from this ideal.
Therefore, while hard soundness serves as an important theoretical standard of correctness, it cannot give nuanced measures of error, particularly in long reasoning chains.
This necessitates more flexible methods for measuring claim soundness even in the presence of errors, which we address next.

\section{Soundness Checks via \OurProb{}}
\label{sec:certifying}

We now consider the practical certification of LLM-generated reasoning chains.
These chains are formed autoregressively: starting from an initial sequence of base claims \(C_1, \ldots, C_n\), the LLM iteratively generates the derived claims \(C_{n+1}, \ldots, C_{n+m}\) where each
\begin{equation*}
    C_{n+k} = \msf{LLM}(C_1, \ldots, C_{n+k-1}),
\end{equation*}
for reasoning steps \(k = 1, \ldots, m\).
We aim to quantify the reliability of this process using a sequence of \textit{entailment stability scores}: \(\tau_1, \ldots, \tau_m \in [0,1]\), where each \(\tau_k\) denotes how reliably the \(k\)-th derived claim \(C_{n+k}\) is entailed with respect to its preceding claims \(C_1, \ldots, C_{n+k-1}\).
The connection between entailment and error detection is straightforward: if \(\tau_k\) is small, then \(C_{n+k}\) is likely an error.



However, a well-principled and computationally tractable formulation of \(\tau_k\) is non-trivial when entailment is probabilistic.
Critically, hard soundness is incompatible with non-binary outputs, and it is not immediately clear how uncertain premises should be evaluated.
\OURS{} addresses this: \cref{sec:motivating_subsets} motivates probabilistic entailment using insights from human psychology, LLM empirics, and mathematical logic.
Subsequently, \cref{sec:res} formalizes our approach, defines \OURS{}, and details its efficient Monte Carlo estimation.

\subsection{Entailment with Probabilistic Premises}
\label{sec:motivating_subsets}

The key challenge lies in accurately assessing entailment when premises are probabilistically uncertain.
Our main insight is to calculate an overall likelihood by averaging across various probable combinations of that uncertain information.

Our approach is motivated by several observations. In \textbf{human cognition}, people naturally discount or ignore dubious statements when reasoning~\citep{johnson2010deductive}. Similarly, lengthy contexts are often filtered to remove irrelevant and erroneous claims to improve \textbf{LLM performance} on reasoning tasks~\citep{mukherjee2025premiseaugmentedreasoningchainsimprove}. These observations collectively motivate our development of a probabilistic entailment framework based on premise subsets.

\begin{figure*}
    \centering
    \includegraphics[width=0.48\linewidth]{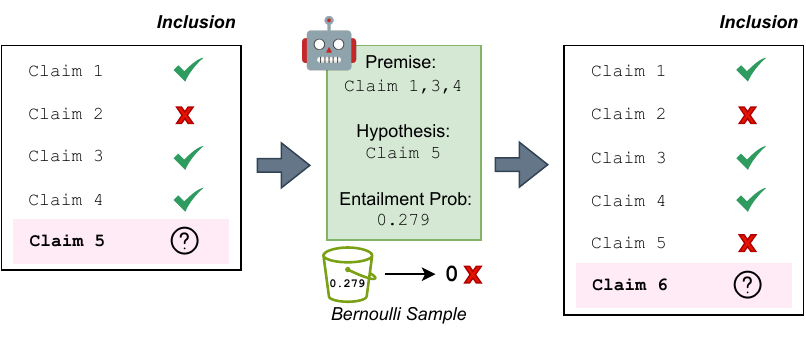}
    \hfill
    \includegraphics[width=0.48\linewidth]{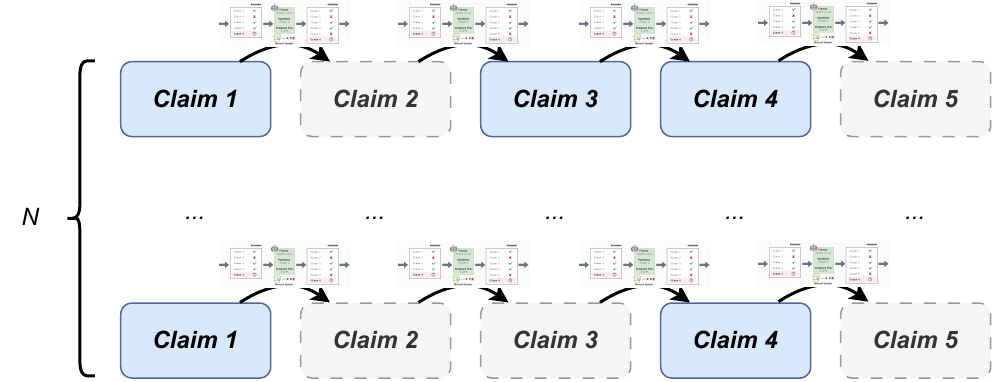}
    \caption{\textbf{(Estimating \OURS{})} 
    (Left)
    The entailment rate of each derived claim is autoregressively computed.
    We first randomly initialize a premise (denoted by \(\alpha\)) according to the base priors \(p_1, \ldots, p_n\).
    Then, for each derived claim, we compute its entailment rate with respect to the premise set.
    Finally, we add this derived claim to the premise set with probability equal to its entailment rate.
    (Right)
    This is run in parallel across \(N\) instances.
    }
    \label{fig:algo}
\end{figure*}

To measure the reliability of a hypothesis \(H\) with respect to a premise \(P\) containing \(k\) claims with uncertain soundness, we consider all \(2^k\) configurations of inclusion and exclusion for \(P\)'s claims.
Each configuration is represented by a binary vector \(\alpha \in \{0,1\}^k\), where \(\alpha_i = 1\) indicates inclusion of claim \(C_i\) and \(\alpha_i = 0\) indicates exclusion.
This leads to the following natural measure of \textit{stability} for \(H\) with respect to \(P\) and \(\mcal{E}\):
\begin{equation}
\label{eqn:tau_generic}
    \tau(\mcal{E}, P, H) = \sum_{\alpha \in \{0,1\}^k} \entail{}(P(\alpha), H) \cdot \Pr[\alpha],
\end{equation}
where \(\Pr[\alpha]\) is the probability of this specific configuration of premise claim inclusions, and depends on the base and derived claims, as well as the entailment model \(\mathcal{E}\), which we discuss next.

\subsection{\OurProb{} with Efficient Sampling}
\label{sec:res}


We previously established a method for calculating the entailment of a single hypothesis based on a set of premises that might be uncertain (\cref{eqn:tau_generic}).
Now, we will extend this concept to evaluate an entire LLM-generated reasoning chain, which consists of multiple steps.
Our goal is to compute a sequence of \textit{entailment stability scores}, denoted as \(\tau_1, \ldots, \tau_m\), where each score $\tau_k$ quantifies the reliability of the $k$-th derived claim, $C_{n+k}$.

The core challenge remains: how to reliably judge a claim when the preceding claims it relies on are themselves not entirely trustworthy?
Our approach, \textbf{\OURS{}}, solves this by autoregressively assessing each claim while accounting for the soundness of previous claims.
In particular, when we evaluate the $k$-th derived claim, we consider all possible combinations of soundness for the preceding $n+k-1$ claims.
The stability score, $\tau_k$, is then the expected entailment of the current claim, averaged across all sound combinations.

To formalize this, we represent a particular combination of inclusion or exclusion of previous claims using a binary vector \(\alpha \in \{0,1\}^{n+k-1}\), where let \(\alpha_i = 1\) denote the inclusion of claim \(C_i\) and let \(\alpha_i = 0\) denote its exclusion.
The probability of this combination \(\Pr[\alpha]\) is calculated recursively as follows:




\begin{itemize}
    \item \textbf{Base Case (\(k = 1\)):} For the first derived claim, $C_{n+1}$, the premises are the initial base claims $C_1, \dots, C_n$.
    We assume that each base claim \(C_i\) is associated with a prior probability of soundness \(p_i\) that is given.
    Therefore, let:
    \begin{equation}
\begin{aligned}
\label{eqn:base_prob}
    & \Pr[\alpha_{1:n}] = \prod_{i = 1}^{n} p_i ^{\alpha_i} (1 - p_i) ^{\alpha_i}
\end{aligned}
\end{equation}

\item \textbf{Inductive Case (\(k > 1\)):}
For subsequent claims, the probability of a specific premise combination $\alpha_{1:n+k}$ depends on two factors: the probability of the preceding combination ($\Pr[\alpha_{1:n+k-1}]$) and the entailment probability of the new claim given that preceding combination.
That is, a claim is added to our set of ``sound'' premises based on how strongly the current set entails it, where let \(\Pr[\alpha_{1:n}] =\)
\begin{equation}
\label{eqn:inductive_prob}
    \Pr[\alpha_{1:n+k-1}] \cdot \mcal{E}(C (\alpha_{1:n+k-1}), C_{n+k})
\end{equation}
where $C(\alpha_{1:n+k-1})$ are the claims indexed by $\alpha_{1:n+k-1} \in \{0,1\}^{n+k-1}$.
\end{itemize}

Using the above definition for \(\Pr[\alpha]\), we may quantify how likely each combination of previous claims may affect the current entailment.
In particular, we naturally define the \textit{entailment stability score} \(\tau_k\) for the \(k\)-th derived claim as a marginalization over all combinations of its predecessors:
\begin{equation} \label{eqn:true_tau_k}
    \tau_{k} = \sum_{\alpha \in \{0,1\}^{n+k-1}} \mcal{E}(C (\alpha), C_{n+k}) \cdot \Pr[\alpha]
\end{equation}
However, directly computing \(\tau_k\) is highly inefficient, as it requires summing over \(2^{n+k-1}\) possible combinations of premise entailment.
Instead, we estimated it by sampling the premise combinations:
\begin{equation} \label{eqn:estimate_tau_k_plus_1}
    \hat{\tau}_{k} = \frac{1}{N} \sum_{i = 1}^{N} \mcal{E}(C(\alpha^{(i)}), C_{n+k}),
\end{equation}
where let \(\alpha^{(1)}, \ldots, \alpha^{(N)} \sim \{0,1\}^{n+k-1}\) be i.i.d. sampled according to~\cref{alg:sampling} 
and in \cref{fig:algo}.
Additionally, note that \(\hat{\tau}_k\) converges rapidly to \(\tau\) as the number of samples \(N\) grows, allowing us to obtain a rigorous statistical guarantee on our stability scores as a function of the number of samples.

\begin{algorithm}[t]
\small
\caption{Estimating \OURS{}
}
\label{alg:sampling}
\begin{algorithmic}[1] 
\Require Reasoning chain $(C_1, \dots, C_{n+m})$, tolerance \((\varepsilon, \delta)\), base priors \(p_1, \ldots, p_n\), and entailment model \(\mcal{E}\).
\State \(N \gets \frac{\log (2m/\delta)}{2\varepsilon^2}\)
\For{$i=1, \ldots, N$}
    \State $\alpha_1^{(i)} \sim \mrm{Bernoulli}(p_1), \ldots, \alpha_n ^{(i)} \sim \mrm{Bernoulli}(p_n)$
    \For{$k=1, \ldots, m$}
        \State $p_{n+k}^{(i)}\gets \mathcal{E}(C (\alpha_{1:n+k-1}^{(i)}), C_{n+k})$
        \State $\alpha_{n+k} ^{(i)} \sim \mrm{Bernoulli}(p_{n+k}^{(i)})$
    \EndFor
\EndFor
\For{$k = 1, \ldots, m$}
    \State $\hat{\tau}_k = \frac{1}{N} \sum_{i = 1}^{N} p_{n+k}^{(i)}$
\EndFor
\end{algorithmic}
\end{algorithm}

\begin{restatable}[Estimating Entailment Stability]{theorem}{sampling}
Let \(N \geq \tfrac{\log(2m/\delta)}{2 \varepsilon^2}\) for any \(\varepsilon > 0\) and \(\delta > 0\).
For any entailment model \(\mcal{E}\) and reasoning chain \((C_1, \ldots, C_{n+m})\), define \(\tau_1, \ldots, \tau_m\) as in~\cref{eqn:estimate_tau_k_plus_1}.
Then, with probability at least \(1 - \delta\), this estimate has error \(\abs{\hat{\tau}_k - \tau_k} \leq \varepsilon\) for all \(k\).
\end{restatable}
\begin{proof}
See~\cref{app:proofs}.
\end{proof}

\paragraph{Error Detection.}
Recall the connection between entailment stability and error detection: the lower a claim's entailment stability \(\tau_k\), the greater its error.
Consider a simple thresholding mechanism: if some estimate \(\hat{\tau}_k\) falls below a prescribed error threshold, then we mark the derived claim \(C_{n+k}\) as erroneous.
In the following, we demonstrate the empirical effectiveness of this procedure.

\section{Evaluating \OURS{} for Estimating Probabilistic Soundness}
\label{sec:experiments}

\begin{table*}[t]
    \centering
    \small
\resizebox{\textwidth}{!}{%
\begin{tabular}{lcccccc}
\toprule
Dataset / Method & \multicolumn{3}{c}{GPT-4o-mini} & \multicolumn{3}{c}{Qwen3-4B} \\
\cmidrule(lr){2-4}
\cmidrule(lr){5-7}
 & Recall & Precision & F1 & Recall & Precision & F1 \\
\midrule
\multicolumn{7}{l}{\textbf{PRMBench}} \\
\rowcolor{gray!20}
\cert{} & \textbf{0.680 $\pm$ 0.024} & \textbf{0.627 $\pm$ 0.021} & \textbf{0.640 $\pm$ 0.023} & \textbf{0.688 $\pm$ 0.020} & \underline{0.623 $\pm$ 0.011} & \underline{0.636 $\pm$ 0.011} \\
\entailprev{} & 0.639 $\pm$ 0.032 & 0.602 $\pm$ 0.016 & 0.596 $\pm$ 0.024 & \textbf{0.698 $\pm$ 0.016} & \underline{0.626 $\pm$ 0.015} & \underline{0.641 $\pm$ 0.017} \\
\entailraw{} & 0.524 $\pm$ 0.022 & 0.511 $\pm$ 0.011 & 0.484 $\pm$ 0.016 & 0.631 $\pm$ 0.016 & 0.558 $\pm$ 0.007 & 0.530 $\pm$ 0.011 \\
ROSCOE-LI-Self & \textbf{0.672 $\pm$ 0.012} & 0.575 $\pm$ 0.007 & 0.489 $\pm$ 0.022 & 0.458 $\pm$ 0.011 & 0.478 $\pm$ 0.006 & 0.446 $\pm$ 0.006 \\
ROSCOE-LI-Source & \textbf{0.676 $\pm$ 0.014} & 0.584 $\pm$ 0.008 & 0.570 $\pm$ 0.011 & 0.497 $\pm$ 0.003 & 0.496 $\pm$ 0.004 & 0.495 $\pm$ 0.004 \\
ReCEval-Intra & 0.563 $\pm$ 0.012 & 0.581 $\pm$ 0.014 & 0.568 $\pm$ 0.013 & 0.550 $\pm$ 0.007 & 0.573 $\pm$ 0.013 & 0.554 $\pm$ 0.007 \\
ReCEval-Inter & \underline{0.664 $\pm$ 0.012} & 0.573 $\pm$ 0.007 & 0.465 $\pm$ 0.022 & 0.449 $\pm$ 0.004 & 0.476 $\pm$ 0.003 & 0.433 $\pm$ 0.004 \\
\llmjudge{} & 0.647 $\pm$ 0.011 & \textbf{0.645 $\pm$ 0.019} & \textbf{0.643 $\pm$ 0.013} & \textbf{0.695 $\pm$ 0.017} & \textbf{0.662 $\pm$ 0.016} & \textbf{0.675 $\pm$ 0.016} \\
\midrule
\multicolumn{7}{l}{\textbf{DeltaBench}} \\
\rowcolor{gray!20}
\cert{} & \textbf{0.702 $\pm$ 0.024} & \textbf{0.728 $\pm$ 0.022} & \textbf{0.708 $\pm$ 0.026} & 0.513 $\pm$ 0.013 & 0.512 $\pm$ 0.013 & 0.498 $\pm$ 0.010 \\
\entailprev{} & \textbf{0.698 $\pm$ 0.032} & \textbf{0.709 $\pm$ 0.029} & \textbf{0.699 $\pm$ 0.031} & 0.523 $\pm$ 0.011 & 0.522 $\pm$ 0.010 & 0.506 $\pm$ 0.009 \\
\entailraw{} & 0.614 $\pm$ 0.010 & 0.596 $\pm$ 0.004 & 0.594 $\pm$ 0.005 & \textbf{0.580 $\pm$ 0.008} & \underline{0.586 $\pm$ 0.008} & \textbf{0.579 $\pm$ 0.009} \\
ROSCOE-LI-Self & 0.579 $\pm$ 0.006 & 0.664 $\pm$ 0.027 & 0.571 $\pm$ 0.013 & \underline{0.555 $\pm$ 0.007} & \textbf{0.638 $\pm$ 0.039} & 0.522 $\pm$ 0.003 \\
ROSCOE-LI-Source & 0.471 $\pm$ 0.006 & 0.456 $\pm$ 0.009 & 0.453 $\pm$ 0.005 & 0.484 $\pm$ 0.013 & 0.472 $\pm$ 0.021 & 0.457 $\pm$ 0.017 \\
ReCEval-Intra & 0.500 $\pm$ 0.000 & 0.357 $\pm$ 0.012 & 0.416 $\pm$ 0.009 & 0.530 $\pm$ 0.006 & 0.529 $\pm$ 0.005 & \underline{0.528 $\pm$ 0.005} \\
ReCEval-Inter & 0.503 $\pm$ 0.007 & 0.508 $\pm$ 0.012 & 0.483 $\pm$ 0.010 & 0.507 $\pm$ 0.006 & 0.508 $\pm$ 0.006 & 0.505 $\pm$ 0.007 \\
\llmjudge{} & 0.498 $\pm$ 0.002 & 0.371 $\pm$ 0.026 & 0.381 $\pm$ 0.027 & \underline{0.548 $\pm$ 0.010} & 0.563 $\pm$ 0.016 & 0.494 $\pm$ 0.009 \\
\midrule
\multicolumn{7}{l}{\textbf{\ourdata{}}} \\
\rowcolor{gray!20}
\cert{} & \textbf{0.914 $\pm$ 0.012} & \textbf{0.921 $\pm$ 0.013} & \textbf{0.903 $\pm$ 0.020} & \textbf{0.731 $\pm$ 0.006} & \underline{0.755 $\pm$ 0.009} & \textbf{0.723 $\pm$ 0.006} \\
\entailprev{} & 0.587 $\pm$ 0.012 & 0.704 $\pm$ 0.025 & 0.491 $\pm$ 0.020 & 0.580 $\pm$ 0.013 & \underline{0.760 $\pm$ 0.006} & 0.480 $\pm$ 0.022 \\
\entailraw{} & 0.645 $\pm$ 0.018 & 0.647 $\pm$ 0.019 & \underline{0.619 $\pm$ 0.021} & \underline{0.586 $\pm$ 0.019} & 0.630 $\pm$ 0.018 & \underline{0.521 $\pm$ 0.026} \\
ROSCOE-LI-Self & 0.528 $\pm$ 0.005 & 0.569 $\pm$ 0.016 & 0.430 $\pm$ 0.011 & 0.568 $\pm$ 0.009 & 0.732 $\pm$ 0.005 & 0.473 $\pm$ 0.017 \\
ROSCOE-LI-Source & 0.540 $\pm$ 0.012 & 0.543 $\pm$ 0.013 & 0.511 $\pm$ 0.016 & 0.491 $\pm$ 0.004 & 0.484 $\pm$ 0.006 & 0.448 $\pm$ 0.008 \\
ReCEval-Intra & 0.500 $\pm$ 0.000 & 0.254 $\pm$ 0.006 & 0.336 $\pm$ 0.005 & 0.500 $\pm$ 0.000 & 0.252 $\pm$ 0.003 & 0.335 $\pm$ 0.003 \\
ReCEval-Inter & 0.546 $\pm$ 0.013 & 0.548 $\pm$ 0.013 & 0.513 $\pm$ 0.016 & 0.495 $\pm$ 0.003 & 0.489 $\pm$ 0.005 & 0.451 $\pm$ 0.007 \\
\llmjudge{} & \underline{0.687 $\pm$ 0.018} & \underline{0.780 $\pm$ 0.016} & \underline{0.628 $\pm$ 0.027} & \underline{0.602 $\pm$ 0.026} & \textbf{0.769 $\pm$ 0.013} & \underline{0.502 $\pm$ 0.034} \\
\midrule
\multicolumn{7}{l}{\textbf{\recipes{}}} \\
\rowcolor{gray!20}
\cert{} & \textbf{0.636 $\pm$ 0.010} & \underline{0.657 $\pm$ 0.011} & \textbf{0.633 $\pm$ 0.010} & \underline{0.532 $\pm$ 0.012} & \underline{0.532 $\pm$ 0.012} & \underline{0.517 $\pm$ 0.009} \\
\entailprev{} & 0.468 $\pm$ 0.004 & 0.462 $\pm$ 0.004 & 0.428 $\pm$ 0.010 & 0.511 $\pm$ 0.005 & \underline{0.529 $\pm$ 0.014} & 0.384 $\pm$ 0.008 \\
\entailraw{} & \underline{0.591 $\pm$ 0.007} & 0.598 $\pm$ 0.008 & \underline{0.589 $\pm$ 0.007} & 0.500 $\pm$ 0.000 & 0.290 $\pm$ 0.005 & 0.367 $\pm$ 0.005 \\
ROSCOE-LI-Self & 0.555 $\pm$ 0.005 & \textbf{0.703 $\pm$ 0.018} & 0.483 $\pm$ 0.011 & \textbf{0.619 $\pm$ 0.007} & \textbf{0.711 $\pm$ 0.012} & \textbf{0.601 $\pm$ 0.010} \\
ROSCOE-LI-Source & 0.500 $\pm$ 0.000 & 0.283 $\pm$ 0.009 & 0.361 $\pm$ 0.007 & 0.500 $\pm$ 0.000 & 0.290 $\pm$ 0.006 & 0.367 $\pm$ 0.004 \\
ReCEval-Intra & 0.515 $\pm$ 0.008 & 0.540 $\pm$ 0.022 & 0.396 $\pm$ 0.010 & 0.500 $\pm$ 0.000 & 0.290 $\pm$ 0.006 & 0.367 $\pm$ 0.004 \\
ReCEval-Inter & 0.500 $\pm$ 0.000 & 0.283 $\pm$ 0.009 & 0.361 $\pm$ 0.007 & 0.500 $\pm$ 0.000 & 0.290 $\pm$ 0.005 & 0.367 $\pm$ 0.004 \\
\llmjudge{} & 0.560 $\pm$ 0.023 & 0.569 $\pm$ 0.024 & 0.530 $\pm$ 0.028 & 0.500 $\pm$ 0.000 & 0.289 $\pm$ 0.005 & 0.366 $\pm$ 0.004 \\
\midrule
\bottomrule
\end{tabular}
}
    \caption{\textbf{(Benchmark Results)} \cert{} is top-performing in majority of settings (5/8), with no other single method being a consistent challenger. For each dataset+model group, \textbf{Bold} is the best and \underline{underline} is the second best.}
    \label{tab:benchmarks-combined}
\end{table*}

\begin{table*}[t]
\scriptsize
\centering
\resizebox{\textwidth}{!}{%
\begin{tabular}{l *{9}{c}}
\toprule
\textbf{Claim} & \makecell[c]{\textbf{ARES}\\\textbf{(Ours)}} & \makecell[c]{Entail\\-Prev} & \makecell[c]{Entail\\-Base} & \makecell[c]{ROSCOE\\-LI-Self} & \makecell[c]{ROSCOE\\-LI-Source} & \makecell[c]{ReCEval\\-Intra} & \makecell[c]{ReCEval\\-Inter} & \makecell[c]{LLM\\-Judge} & \makecell[c]{\textit{Ground}\\\textit{Truth}} \\
\midrule
\multicolumn{10}{p{\dimexpr\linewidth-2\tabcolsep}}{\textbf{Context} Rules: H3 -> AZ; SG -> C6; C6 -> GM; VD -> H3; G8 -> VD; D8 -> U8; U8 -> DG; DG -> G8. Fact: I have D8. ...}  \\
\cmidrule{1-10}
Claim 5: I use rule (VD -> H3) to derive H3 & \textbf{0.79}\greencheck & \textbf{1.00}\greencheck & 0.00\redx & \textbf{1.00}\greencheck & 0.00\redx & \textbf{1.00}\greencheck & 0.00\redx & \textbf{1.00}\greencheck & \greencheck \\
\cmidrule{1-10}
Claim 6: I use rule (H3 -> AZ) to derive AZ & \textbf{0.82}\greencheck & \textbf{1.00}\greencheck & \textbf{1.00}\greencheck & \textbf{1.00}\greencheck & \textbf{1.00}\greencheck & \textbf{1.00}\greencheck & \textbf{1.00}\greencheck & \textbf{1.00}\greencheck & \greencheck \\
\cmidrule{1-10}
Claim 7: I use rule (AZ -> SG) to derive SG & \textbf{0.00}\redx & \textbf{0.00}\redx & \textbf{0.00}\redx & 1.00\greencheck & \textbf{0.00}\redx & 1.00\greencheck & \textbf{0.00}\redx & \textbf{0.00}\redx & \redx \\
\cmidrule{1-10}
Claim 8: I use rule (SG -> C6) to derive C6 & \textbf{0.00}\redx & 1.00\greencheck & \textbf{0.00}\redx & 1.00\greencheck & \textbf{0.00}\redx & 1.00\greencheck & \textbf{0.00}\redx & 1.00\greencheck & \redx \\
\bottomrule
\end{tabular}
}
\caption{In this ClaimTrees example, after two correct steps (\textbf{Claims 5–6}), an initial error (\textbf{Claim 7}) using the non-existing rule AZ → SG causes a propagated error (\textbf{Claim 8}). Only \OURS{} correctly judges all steps.}
\label{tab:claimtrees-example}
\end{table*}

\OURS{} performs error detection by estimating the entailment stability of each derived claim and applying a thresholding mechanism.
We next run experiments to validate the performance of \OURS{} against multiple baselines on diverse benchmarks.


\paragraph{Experiment Setup.} We consider comparisons with \llmjudge{}, which takes the whole reasoning chain as input and makes a judgment for each step together, \entailprev{} and \entailraw{}, which judge the entailment of a claim based on all preceding claims and only base claims respectively, and two ROSCOE~\citep{golovneva2023roscoe} and two ReCEval~\citep{prasad-etal-2023-receval} correctness methods that are based on pairwise comparisons.

For LLMs, we use GPT-4o-mini~\citep{openai2024gpt4oMini} and Qwen3-4B~\citep{qwen3}.
For a PRM, we used Qwen2.5-Math-PRM-7B~\citep{zhang2025lessonsdevelopingprocessreward}.
We evaluate on four datasets: PRMBench~\citep{song2025prmbenchfinegrainedchallengingbenchmark}, DeltaBench~\citep{he2025largelanguagemodelsdetect}, \ourdata{} (our synthetic data), and \recipes{} (graph-based recipe dataset derived from CaptainCook4D~\citep{peddi2024captaincook}).
We evaluate using Macro-recall, Macro-precision and Macro-F1 following the literature~\citep{he2025largelanguagemodelsdetect}.
To compute the error threshold for the entailment scores, we first sweep over all the values that occur in the training split and select the one that maximizes Macro-F1.
We repeat this process in a 5-fold cross-validation where each time we use one fold for validation and four folds for testing and report the average and standard deviation.
Additional details and analyses can be found in~\cref{app:experiments}.

\subsection{RQ1: Does \cert{} work better than baseline methods on Benchmarks?}
We measure \cert{}'s ability to identify errors in natural reasoning chains using PRMBench and DeltaBench.
With GPT-4o-mini backbone entailment model, we find that \OURS{} achieves the best Macro-F1 scores on both datasets, shown in~\cref{tab:benchmarks-combined}.
\llmjudge{} performs poorly on DeltaBench while \entailraw{} underperforms on PRMBench.
DeltaBench's long reasoning chains appear to confuse LLM-Judge when making holistic judgments. For Qwen3-4B, \entailraw{} performs slightly better, while all other methods lag behind.
Our inspection reveals that Qwen3-4B-based entailment models frequently classify next claims as entailed, suggesting limited capability for judging complex reasoning.
Additional experiments in \cref{app:prm} show that \cert{} can also achieve further improvements on top of a PRM backbone.

\subsection{RQ2: In what setting does \cert{} identify more errors than baselines?}
To pinpoint where ARES most effectively outperforms other methods, we needed to test it in settings with long reasoning chains and clear error propagation. Since existing benchmarks often lack these specific features, we constructed two controllable datasets designed to isolate these challenges:
\begin{itemize}
\item \textbf{\ourdata{}:} A synthetic logical reasoning dataset involving proofs over abstract graphs.
\item \textbf{\recipes{}:} A graph-rule-based dataset adapted from the cooking task graphs in CaptainCook4D~\citep{peddi2024captaincook}.
\end{itemize}
We designed these datasets specifically to test the core reasoning capabilities of each method, controlling for confounding variables. For example, \ourdata{} uses abstract symbols and shuffled rules to mitigate ordering bias. In both datasets, we represent the underlying rules (e.g., logical rules, recipe actions) as base claims. We then intentionally remove a key base claim—like a rule in a proof or an ingredient in a recipe—to create unsound derivations, allowing us to precisely track how the initial error propagates through the reasoning chain. Further details can be found in Appendix~\ref{app:synth_datasets}.

Experiments on these controlled datasets confirm that ARES excels at identifying propagated errors, especially in long chains. As demonstrated in Figure~\ref{fig:long_chains}, ARES maintains a high Macro-F1 score even as chains become very long, whereas the performance of all baseline methods deteriorates sharply after just a few steps. For example, \cert{} sustains a Macro-F1 score of at least 89\% on chains up to 50 steps long, while other methods fall into the 30–40\% range. The results in Table~\ref{tab:benchmarks-combined} further highlight this robust performance across our synthetic datasets, with an example shown in Table~\ref{tab:claimtrees-example}.
We further discuss in \cref{app:method} that only \OURS{} satisfies all important desiderata for error detection while other methods fail to.

\begin{figure}[t]
\centering
\includegraphics[width=\linewidth]{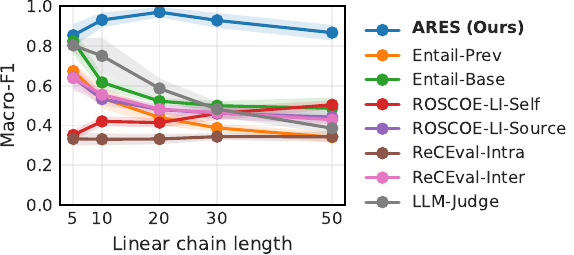}
\caption{\textbf{(\ourdata{}) GPT-4o-mini.} \cert{} can robustly identify error propagations in long reasoning chains, whereas other methods fail.}
\label{fig:long_chains}
\end{figure}
\begin{figure}
    \centering
    \includegraphics[width=0.85\linewidth]{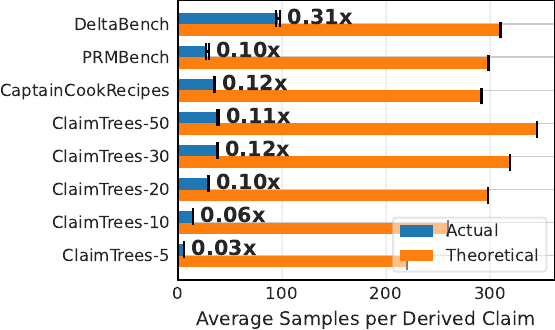}
    \caption{\textbf{(Per-Claim Samples)} \cert{} in practice only uses 0.03x to 0.31x the number of samples required by the theoretical bound.}
    \label{fig:per-step-samples}
\end{figure}

\begin{table*}[t]
  \centering
  \small
\begin{tabular}{lcccc}
\toprule
Method & \textbf{PRMBench} & \textbf{DeltaBench} & \textbf{\ourdata{}-10} & \textbf{\recipes{}} \\
\midrule
\cert{}-$\varepsilon 0.1$ & \textbf{0.640} & \textbf{0.708} & \textbf{0.931} & \underline{0.633} \\
\cert{}-$\varepsilon 0.2$ & \underline{0.599} & \underline{0.697} & \underline{0.926} & 0.631 \\
\cert{}-$\varepsilon 0.3$ & 0.582 & 0.694 & 0.919 & 0.621 \\
\cert{}-$\varepsilon 0.4$ & 0.595 & 0.687 & 0.922 & \textbf{0.640} \\
\bottomrule
\end{tabular}
\caption{\textbf{(GPT-4o-mini) Performance Convergence with Samples} \cert{} is able to achieve high accuracy even when using a smaller number of samples. When $\varepsilon=$0.1, 0.2, 0.3, 0.4, a sequence of length $m=10$ needs 265, 67, 30, 17 samples per step respectively. We can see that there is no significant performance change when we increase the $\epsilon$ to 0.4 and thus decrease the number of samples 15x.}
  \label{tab:efficiency-convergence}
\end{table*}

\subsection{RQ3: Is \OURS{} computationally efficient?}

While \cert{} samples multiple combinations of previous claims to check soundness, it is implemented efficiently to avoid redundant LLM calls for the same premise-hypothesis pairs. Figure~\ref{fig:per-step-samples} shows the theoretical versus actual samples used for each derived claim and complete reasoning chain, respectively.
For \ourdata{}, the average total samples per example increases with chain length. With shorter chains (\ourdata{}-5), we achieve extreme efficiency at only 0.03x of theoretical samples needed. DeltaBench uses the most samples but still achieves 0.31x of theoretical samples needed.
DeltaBench needing more samples indicates greater uncertainty in the entailment model's outputted probability for this dataset.
In ideal cases where the entailment model outputs only 1 or 0 for every derived claim, we need just one sample per claim each step.

To have a more direct analysis of efficiency, we conduct another analysis on the performance vs. sample size trade-off across all datasets on GPT-4o-mini in~\cref{tab:efficiency-convergence}. We find that in practice \cert{}’s performance is stable even with fewer samples, indicating potential further computational savings.
On synthetic benchmarks \ourdata{} and \recipes{}, there is no significant difference for $\varepsilon=0.1$ to $0.4$, while more differences are shown for PRMBench and DeltaBench.

\subsection{RQ4: Is \OURS{} useful for selecting Best-of-N generations?}
\label{sec:best-of-n}
To test if \cert{} is useful for downstream tasks, we run a best-of-N experiment--selecting the generation scoring the highest in soundness among multiple generations, and see which methods' selected generations have better accuracies.
We perform the experiment on PRMBench, which contains both original and modified process. We use both as the two generations, with the original process as the correct generation and the modified process as the incorrect generation.
We select the best of two candidate chains by either averaging all step scores or using only the final step’s score.

\begin{table}[t]
    \centering
    \small
\begin{tabular}{lcc}
\toprule
Method & Step Avg & Final Step\\
\midrule
\cert{} & \textbf{0.730} & \textbf{0.660}\\
\entailprev{} & \textbf{0.790} & 0.240\\
\entailraw{} & 0.540 & 0.300\\
ROSCOE-LI-Self & 0.540 & 0.210\\
ROSCOE-LI-Source & 0.630 & 0.310\\
ReCEval-Intra & 0.480 & 0.060\\
ReCEval-Inter & 0.480 & 0.190\\
LLM-Judge & 0.570 & 0.250\\
\bottomrule
\end{tabular}
    \caption{\textbf{(PRMBench Best-of-N)} \cert{} is a robust predictor for downstream task performance. The table shows the selection accuracy (higher is better) for choosing the correct reasoning chain from two options. Using the final step's score is a stricter evaluation, where \cert{}'s performance stands out. \textbf{Bold} indicates the best performance within bootstrap standard error.}
    \label{tab:best-of-n-prmbench}
\end{table}

Results in \cref{tab:best-of-n-prmbench} show that when using the score of the final step---a stricter and often more decisive measure---\cert{} significantly outperforms all other methods. Notably, the performance of simpler methods like \entailprev{} collapses on this stricter metric. This highlights \cert{}’s unique strength in maintaining a sound evaluation of steps throughout the entire reasoning chain, making its final assessment particularly reliable.
Therefore, \cert{} is a strong and robust predictor of downstream task performance.

\subsection{Ablations}
\label{sec:ablations}

We conduct ablations on \ourdata{} to examine the robustness and design choices of \cert{}.

\paragraph{Robustness to Irrelevant Claims.} We tested our method on reasoning trees with varying widths (irrelevant sources) and depths (path lengths), where an error was introduced by removing a single rule. As shown in \cref{tab:synthdata_wide}, ARES remains stable across all configurations. In contrast, baseline methods degrade, suffering more from increased depth than width. 
\OURS{} is thus capable of filtering irrelevant claims, and error propagation in long chains is the primary reason other approaches fail.



\paragraph{Base Claim Inclusion Probability.}
We also vary the probability $p$ of including base claims and compare probabilistic vs.\ binary entailment models. Results in \cref{tab:synthdata_long} show that $p=1$ with a probabilistic model consistently performs best, while binary models sometimes benefit from $p<1$. Choosing $p=1$ is therefore often both accurate and efficient, as it avoids resampling base claims and reduces variance. However, as \ourdata{} has a clearer cut in soundness, the case can be different when entailment contains ambiguity.

\paragraph{Additional Results.}
We further examine benign errors (inserted rules that do not affect downstream steps) in \cref{app:ablations}. All methods perform equally well, unlike the irrelevant-claim setting where baselines degrade. 

\subsection{Discussion of Errors}

Our inspection of the data and error detection outputs reveals some insights.
\entailraw{} fails on PRMBench because judging entailment in long math derivations is challenging. Both \llmjudge{} and \entailraw{} fail in DeltaBench, with \entailraw{} struggling to judge entailment in very long reasoning chains.
In naturally occurring datasets, error propagation is limited and not always annotated, so \entailprev{} performs close to \cert{}. However, synthetic data shows \entailprev{} fails with propagated errors.
\llmjudge{} sometimes fails to follow instructions, outputting incorrect numbers of scores relative to claims being judged.
Pairwise methods in ROSCOE and ReCEval cannot detect complex errors that need multiple claims as premise.
\cert{} can only improve upon entailment models that can already do correct entailment.

\section{Related Work}
\label{sec:related}

\paragraph{Guarantees for Single-Step Explanations.} Research in interpretability has shifted from heuristic evaluation toward formal guarantees for individual predictions. One major branch is inherently interpretable models, which provide guarantees such as optimality~\citep{angelino2018learning,ustun2019learning}, monotonicity~\citep{gupta2016monotonic,milani2016fast}, or faithfulness by construction in deep learning models~\citep{you2025sumofparts,bassan2025explain}. 
A second branch focuses on post-hoc explanations for black-box models, including conservation guarantees~\citep{bach2015pixel,shrikumar2017learning,montavon2017explaining}, local accuracy, missingness, and consistency~\citep{lundberg2017unified,wu2024discret}, precision~\citep{ribeiro2018anchors}, minimality~\citep{ferreira2022looking,bassan2023towards,bassan2025explain}, sufficiency~\citep{bassan2025explain,bassan2025explaining}, or certified interventions via recourse methods~\citep{ustun2019actionable,karimi2020model}.
A third branch of work provides certified robustness guarantees, particularly in the form of \textit{stability certificates}~\citep{xue2023stability,kim2024evaluating,jin2025probabilistic}, which have been applied to model explainability in medicine~\citep{achara2025invisible}.
Our work extends stability guarantees to LLM reasoning chains.


\paragraph{Guarantees and Verification for Multi-Step Reasoning.} While single-step methods are well-studied, LLMs often generate multi-step reasoning chains prone to \textit{hallucinations} and \textit{error propagation}~\citep{huang2025survey,lyu-etal-2024-towards}. A significant body of work focuses on practical error detection without formal guarantees, including self-consistency checkers~\citep{manakul2023selfcheckgpt,dhuliawala2023chain} and automated verifiers such as LLM Judges~\citep{tyagi-etal-2024-step,he2024socrevallargelanguagemodels,he2025largelanguagemodelsdetect}, Process Reward Models (PRMs)~\citep{lightman2023letsverifystepstep,zhang2025lessonsdevelopingprocessreward}, and specialized entailment models~\citep{dalvi-etal-2021-explaining,weir-etal-2024-enhancing,havaldar2025entailedlinesincorporatingimplication}.

To provide more rigor, logic-based verifiers assess \textit{soundness}, though often limiting to pairwise checks~\citep{golovneva2023roscoe,prasad-etal-2023-receval} or taking a brittle approach to propagated errors~\citep{mukherjee2025premiseaugmentedreasoningchainsimprove}. An early formal guarantee, Faithful Chain-of-Thought, ensures reasoning traces deterministically yield the final answer~\citep{lyu2023faithful}. While a suite of evaluation benchmarks exists~\citep{tyagi-etal-2024-step,jacovi-etal-2024-chain,song2025prmbenchfinegrainedchallengingbenchmark,zheng2024processbenchidentifyingprocesserrors,he2025largelanguagemodelsdetect}, a unified standard for error definition is still emerging~\citep{lee2025evaluatingstepbystepreasoningtraces,mukherjee2025premiseaugmentedreasoningchainsimprove}.
Recent statistical methods provide calibrated step-level reliability for generation but focus on isolated predictions within the chain~\citep{feng2025bird,quach2023conformal,cherian2024large}. In contrast, our work introduces \textit{propagation-aware guarantees} that certify \textit{entire reasoning chains}, ensuring upstream errors do not corrupt downstream judgments.


\section{Conclusion}
\label{sec:conclusion}
Current methods cannot reliably detect propagation errors in LLM reasoning chains.
To address this limitation, we introduce \OURS{}, a novel framework for certifying the soundness of an LLM's reasoning chain.
By quantifying the soundness of each claim through autoregressive sampling, \OURS{} provides a fine-grained inductive guarantee on the chain's overall reliability that is useful in error detection.
Empirically, \OURS{} demonstrates superior performance, robustly identifying errors in lengthy and complex reasoning chains where existing methods fail due to error propagation.

\section*{Limitations}

ARES's performance is tied to the quality of the entailment model; poor calibration can lead to unreliable scores.
However, our model-agnostic approach allows for easily substituting better-calibrated components via techniques like temperature scaling to improve performance.

While our efficient sampling algorithm mitigates computational overhead, ARES is more intensive than simpler approaches.
Additionally, our approach assumes that the claims are already decomposed and, therefore, cannot detect errors at the sub-claim level.
We leave this for future work, noting it would increase computational costs.

Finally, our evaluation on four datasets with two models demonstrates effectiveness across different domains, but it is not exhaustive. Performance could also be improved with better LLM prompts.

\section*{Ethical Considerations}
Our research framework for detecting reasoning errors raises several ethical considerations. While \OURS{} can improve the reliability of AI reasoning, it may create false confidence in underlying models when they consistently make undetected errors. Implementation requires careful evaluation across diverse domains to prevent biases from propagating through certified reasoning chains. Additionally, computing resource requirements for probabilistic sampling may limit accessibility to well-resourced institutions. We acknowledge the importance of transparent reporting of \OURS{}'s limitations and recommend human oversight when used in high-stakes domains such as healthcare or legal applications to ensure responsible deployment.

\section*{Potential Risks}
While \OURS{} offers significant advantages over existing methods, there are several potential risks to consider. First, the probabilistic sampling approach introduces computational overhead, though our efficient algorithm mitigates this. Second, \OURS{} may create false confidence in underlying LLM reasoning when it consistently fails to detect certain types of errors. Implementation requires careful evaluation across diverse domains to prevent biases from propagating through certified reasoning chains. Additionally, computing resources for probabilistic sampling may limit accessibility to well-resourced institutions. Human oversight remains essential when deployed in high-stakes domains like healthcare or legal applications to ensure responsible use and reliable reasoning verification.
\section*{Acknowledgments}
We thank Mayank Keoliya for discussion about this work.
This research was partially supported by a gift from AWS AI to Penn Engineering's ASSET Center for Trustworthy AI, by ASSET Center Seed Grant, ARPA-H program on Safe and Explainable AI under the award D24AC00253-00, by NSF award CCF 2442421, by the AI2050 program at Schmidt Sciences (Grant G-25-67983), and by funding from the Defense Advanced Research Projects Agency's (DARPA) SciFy program (Agreement No. HR00112520300). The views expressed are those of the author and do not reflect the official policy or position of the Department of Defense or the U.S. Government.

\bibliography{main}

\appendix

\renewcommand{\thetable}{A\arabic{table}}
\renewcommand{\thefigure}{A\arabic{figure}}

\section{Proofs}
\label{app:proofs}

\sampling*
\begin{proof}
Let $\mathcal{A}_i$ denote the event that $|\hat{\tau}_i - \tau_i| < \varepsilon$ for each $i \in \{n+1,\ldots,n+m\}$.
We want to prove that
\begin{equation}
    \prob{}\left(\bigcap_{i={n+1}}^{n+m} \mathcal{A}_i\right) = 1 - \prob{} \left(\bigcup_{i={n+1}}^{n+m} \bar{\mathcal{A}}_i\right)\geq 1 - \delta.
\end{equation}
According to Boole's inequality and Hoeffding's inequality, 
\begin{align}
    &\prob{}\left(\bigcup_{i={n+1}}^{n+m} \bar{\mathcal{A}}_i\right) \leq \sum_{i={n+1}}^{n+m} \prob{}(\bar{\mathcal{A}}_i)  \tag{Boole's}\\
    &=  \sum_{i={n+1}}^{n+m} \prob{}(|\hat{\tau}_i - \tau_i| \geq \varepsilon) \\
    &\leq \sum_{i={n+1}}^{n+m} 2\exp(-2N\varepsilon^2)\tag{Hoeffding's} \\
    &= 2m\exp(-2N\varepsilon^2)\\
    &\leq \delta\quad\text{when }N\geq \frac{\log (2m/\delta)}{2\varepsilon^2},
\end{align}
with the estimation error of each stability rate bounded by $\delta_i=\frac{\delta}{m}$.
\end{proof}
\section{Method}
\label{app:method}

There are three important desiderata for error detection methods:
\begin{enumerate}
    \item \textbf{Robust:} Previous errors do not adversely affect current step.
    \item \textbf{Causal:} Downstream steps do not affect current step. 
    \item \textbf{Sufficient:} All relevant claims included as premise for detection.
\end{enumerate}

\cref{tab:methods_desiderata} shows that only \OURS{} satisfies all desiderata while none of the baseline methods does.

\begin{table*}[t]
  \centering
  \begin{minipage}{\textwidth}
    \centering
    \begin{tabular}{l ccc}
      \toprule
      \textbf{Method} & \textbf{Robust} & \textbf{Causal} & \textbf{Sufficient} \\ \midrule
      \textbf{\cert{} (ours)} & \cmark & \cmark & \cmark \\
      \entailprev{} & \xmark & \cmark & \cmark  \\
      \entailraw{} & \cmark & \cmark & \xmark \\
      ROSCOE-LI-Self & \xmark & \cmark & \xmark  \\
      ROSCOE-LI-Source & \xmark & \cmark & \xmark \\
      ReCEval-Intra & \cmark & \cmark & \xmark  \\
      ReCEval-Inter& \xmark & \cmark & \xmark \\
      \llmjudge{} & \xmark & \xmark & \cmark\\ \bottomrule
    \end{tabular}
    \caption{(Desiderata for methods) \textbf{Robust:} Previous errors do not adversely affect current step. \textbf{Causal:} Downstream steps do not affect current step. \textbf{Sufficient:} All relevant claims included as premise for detection.}
    \label{tab:methods_desiderata}
  \end{minipage}\hfill
\end{table*}

\section{Experiments}
\label{app:experiments}

\subsection{Entailment Model}
We instantiate an LLM judge to assess whether a hypothesis \(h\) is supported by a premise \(P\) that may contain multiple claims. We use two output modes:
(i) \emph{binary} YES/NO, mapped to \(\{1,0\}\);
and (ii) a \emph{7-point Likert scale} where the LLM must output exactly one label from
\{\textit{Very Likely}, \textit{Likely}, \textit{Somewhat Likely}, \textit{Neutral}, \textit{Somewhat Unlikely}, \textit{Unlikely}, \textit{Very Unlikely}\}.
We convert the label to a probability via
\[
\begin{aligned}
\text{Very Likely} &\mapsto 1.0\\
\text{Likely} &\mapsto 0.8\\
\text{Somewhat Likely} &\mapsto 0.6\\
\text{Neutral} &\mapsto 0.5\\
\text{Somewhat Unlikely} &\mapsto 0.4\\
\text{Unlikely} &\mapsto 0.2\\
\text{Very Unlikely} &\mapsto 0.0
\end{aligned}
\]
which fits in a double-column layout.

\noindent\textbf{Contradiction scoring.}
For contradiction judgments (e.g., in ROSCOE and ReCEVal), we apply the same labels but invert the mapping so that higher scores indicate stronger contradiction:
\[
\begin{aligned}
\text{Very Unlikely} &\mapsto 1.0\\
\text{Unlikely} &\mapsto 0.8\\
\text{Somewhat Unlikely} &\mapsto 0.6\\
\text{Neutral} &\mapsto 0.5\\
\text{Somewhat Likely} &\mapsto 0.4\\
\text{Likely} &\mapsto 0.2\\
\text{Very Likely} &\mapsto 0.0
\end{aligned}
\]
with the binary case mapped analogously (YES/NO \(\mapsto\) \(1/0\) for ``is contradiction?'').

\subsection{Hyperparameters for \cert{}}
\label{app:hyperparams}
In our experiments, we used $\delta=0.1$ and $\varepsilon=0.1$ for \cert{}, which determines the number of samples to take.
We use $p=0.95$ for the inclusion rate for base claims to allow buffer for information overload.
The hyperparameters $\delta=0.1$ and $\varepsilon=0.1$ are chosen following previous work~\citep{jin2025probabilistic}.
The prompts are tuned by manually examining the results and seeing that they are able to produce reasonable results for entailment on each dataset.
We observe that DeltaBench prefers simpler and more natural prompts while for other datasets the prompts need to have more specifications for what entail vs. not entail mean.
The prompts are released in the codebase \url{https://github.com/fallcat/ares}.

\subsection{Experiment Details}
\label{app:experiment_detials}
We use a subset of examples for each experiment.
Experiment results are computed using 5-fold cross-validation.
For each split, the thresholds are picked for the best Macro-F1 on the validation split, and the final numbers are on the test split, averaged over the 5 folds.
The standard deviation is reported for the 5 folds.
Specific packages used can be found in our codebase \url{https://github.com/fallcat/ares}.

\subsection{Licenses for Artifacts}
All datasets, models, and code used in this work follow the licenses and terms specified by their original authors, as cited in the corresponding papers. We do not redistribute these artifacts under different terms. For the artifacts we release (\url{https://github.com/fallcat/ares}) code and evaluation scripts), we provide them under the MIT license.

Our use of existing artifacts is consistent with their intended use as specified by the original authors (e.g., datasets accessed for research purposes were used only in research contexts). For the artifacts we create, we specify that they can be freely used, modified, and redistributed under the MIT license.

\subsection{Controllable Datasets}
\label{app:synth_datasets}

\begin{table*}[t]
    \centering
    \small
\begin{tabular}{lcc}
\toprule
Method & Using Step Average (acc$\pm$std) & Using Final Step (acc$\pm$std) \\
\midrule
\rowcolor{gray!20}
\OURS{} & \textbf{0.730$\pm$0.045} & \textbf{0.660$\pm$0.049}\\
\entailprev{} & \textbf{0.790$\pm$0.043} & 0.240$\pm$0.042\\
\entailraw{} & 0.540$\pm$0.049 & 0.300$\pm$0.046\\
ROSCOE-LI-Self & 0.540$\pm$0.051 & 0.210$\pm$0.041\\
ROSCOE-LI-Source & 0.630$\pm$0.049 & \underline{0.310$\pm$0.043}\\
ReCEval-Intra & 0.480$\pm$0.050 & 0.060$\pm$0.024\\
ReCEval-Inter & 0.480$\pm$0.048 & 0.190$\pm$0.038\\
\llmjudge{} & 0.570$\pm$0.050 & 0.250$\pm$0.044\\
\bottomrule
\end{tabular}
    \caption{\textbf{(PRMBench Best-of-N)} \cert{} is a strong and robust predictor of downstream task performance. \textbf{Bold} is the best and \underline{underline} is the second best.}
    \label{tab:best-of-n-prmbench-all}
\end{table*}

\paragraph{\ourdata{}.}
\ourdata{} is a synthetic dataset in which the reasoning chain reasons starts from a state such as AZ, and reason all the way to another state, say VD. All the reasoning rules are provided in the premise, except one, so that from that point on we know that all the claims are unsound:
An example of a chain of reasoning is shown in \cref{fig:claimtree-5-example}.
In this example, rule H3 -> VD does not actually exist, and thus the reasoning steps starting from the third derived step onward are unsound claims.
We can construct reasoning chains with arbitrary length and errors occurring at different places.

\paragraph{\recipes{}.}
\recipes{} is derived from the recipe graphs in CaptainCook4D~\citep{peddi2024captaincook}, where certain actions must follow other actions.
We then construct base claims using edges in the graph as rules, similar to how we construct the ones in \ourdata{}.
In addition, we add ingredients to the base claims and randomly drop an ingredient.
Then, all the claims that require the ingredient and claims that follow them become unsound.
We extract the ingredients from the claims using GPT-4o-mini.

\begin{figure}[t]
    \centering
    \begin{tcolorbox}[title={Long Chain Example.}]
\textbf{Base Claims:} \\
Rule: AZ -> DG (meaning that if I have AZ, I can derive DG) \\
Rule: SG -> H3 (meaning that if I have SG, I can derive H3) \\
I have AZ \\
Rule: DG -> SG (meaning that if I have DG, I can derive SG) \\
\textbf{Reasoning Steps:} \\
I have AZ, I use rule (AZ -> DG) to derive DG, now I have DG \\
I have DG, I use rule (DG -> SG) to derive SG, now I have SG \\
I have SG, I use rule (SG -> H3) to derive H3, now I have H3 \\
I have H3, I use rule (H3 -> VD) to derive VD, now I have VD
\end{tcolorbox}
    \caption{Long Chain Example for \ourdata{}}
    \label{fig:claimtree-5-example}
\end{figure}

An example of results for \recipes{} is shown in \cref{tab:recipe-example}.
With propagated errors present, only \OURS{} is able to capture all errors.

\subsection{Computing Resources}
We used an NVIDIA A100 GPU with 80GB of memory for the Qwen3-4B model.
For GPT-4o-mini, we used approximately 600 USD in total for prototyping and experiments.


\subsection{Best-of-N Results}
For best-of-N result with standard deviations, see \cref{tab:best-of-n-prmbench-all}.

\subsection{Ablations}
\label{app:ablations}

On \ourdata{}, we also construct two other types of trees to inspect the strengths of \cert{}: wide trees with more sources, imitating the behavior of inserting irrelevant claims, and trees with inserted outgoing edges that are not in the base claims, imitating the case of benign errors in the reasoning chains that do not result in error propagation.

\begin{table*}[t]
    \centering
    \small
\begin{tabular}{lccc}
\toprule
Dataset / Method & \multicolumn{3}{c}{Qwen2.5-Math-PRM-7B} \\
\cmidrule(lr){2-4}
 & Recall & Precision & F1 \\
\midrule
\multicolumn{4}{l}{\textbf{PRMBench}} \\
\rowcolor{gray!20}
\cert{} & \textbf{0.751 $\pm$ 0.017} & \textbf{0.733 $\pm$ 0.020} & \textbf{0.736 $\pm$ 0.014} \\
\entailprev{} & \textbf{0.751 $\pm$ 0.016} & \textbf{0.733 $\pm$ 0.020} & \textbf{0.736 $\pm$ 0.013} \\
\entailraw{} & 0.643 $\pm$ 0.022 & 0.632 $\pm$ 0.024 & 0.624 $\pm$ 0.018 \\
ROSCOE-LI-Self & 0.651 $\pm$ 0.013 & 0.598 $\pm$ 0.013 & 0.592 $\pm$ 0.006 \\
ROSCOE-LI-Source & 0.670 $\pm$ 0.020 & 0.621 $\pm$ 0.019 & 0.623 $\pm$ 0.013 \\
ReCEval-Inter & 0.644 $\pm$ 0.014 & 0.597 $\pm$ 0.013 & 0.596 $\pm$ 0.009 \\
\prm{} & \textbf{0.763 $\pm$ 0.020} & \textbf{0.743 $\pm$ 0.017} & \textbf{0.749 $\pm$ 0.016} \\
\midrule
\multicolumn{4}{l}{\textbf{\ourdata{}-10}} \\
\rowcolor{gray!20}
\cert{} & \textbf{0.739 $\pm$ 0.013} & \textbf{0.743 $\pm$ 0.012} & \textbf{0.733 $\pm$ 0.010} \\
\entailprev{} & \underline{0.722 $\pm$ 0.016} & \underline{0.725 $\pm$ 0.017} & \underline{0.715 $\pm$ 0.011} \\
\entailraw{} & 0.611 $\pm$ 0.013 & 0.616 $\pm$ 0.013 & 0.597 $\pm$ 0.017 \\
ROSCOE-LI-Self & 0.655 $\pm$ 0.005 & 0.662 $\pm$ 0.005 & 0.644 $\pm$ 0.008 \\
ROSCOE-LI-Source & 0.604 $\pm$ 0.020 & 0.612 $\pm$ 0.020 & 0.591 $\pm$ 0.024 \\
ReCEval-Inter & 0.629 $\pm$ 0.020 & 0.628 $\pm$ 0.019 & 0.624 $\pm$ 0.020 \\
\prm{} & 0.607 $\pm$ 0.012 & 0.622 $\pm$ 0.013 & 0.594 $\pm$ 0.017 \\
\midrule
\multicolumn{4}{l}{\textbf{CaptainCook4D}} \\
\rowcolor{gray!20}
\cert{} & \textbf{0.551 $\pm$ 0.012} & \textbf{0.556 $\pm$ 0.014} & \textbf{0.543 $\pm$ 0.012} \\
\entailprev{} & \textbf{0.553 $\pm$ 0.011} & \textbf{0.560 $\pm$ 0.014} & \textbf{0.546 $\pm$ 0.010} \\
\entailraw{} & 0.531 $\pm$ 0.016 & 0.533 $\pm$ 0.017 & 0.519 $\pm$ 0.014 \\
ROSCOE-LI-Self & \underline{0.546 $\pm$ 0.008} & \textbf{0.563 $\pm$ 0.016} & 0.529 $\pm$ 0.008 \\
ROSCOE-LI-Source & 0.469 $\pm$ 0.015 & 0.464 $\pm$ 0.018 & 0.457 $\pm$ 0.017 \\
ReCEval-Inter & 0.469 $\pm$ 0.015 & 0.465 $\pm$ 0.018 & 0.461 $\pm$ 0.017 \\
\prm{} & \textbf{0.560 $\pm$ 0.013} & \textbf{0.569 $\pm$ 0.017} & \textbf{0.552 $\pm$ 0.013} \\
\midrule
\bottomrule
\end{tabular}
    \caption{\textbf{(Benchmark Results on Qwen2.5-Math-PRM-7B)} \cert{} performs the best across various datasets and backbone entailment models. For each dataset+model group, \textbf{Bold} is the best and \underline{underline} is the second best.}
    \label{tab:benchmarks-qwenprm}
\end{table*}

\paragraph{Inserting Irrelevant Claims.}
We investigated whether \cert{} better identifies errors in long versus wide chains with the same number of nodes. We constructed wide reasoning trees with multiple sources and one sink, with a rule removed in the middle. Starting from one source, we derive to the sink node, where the rule error can be in the path from source to sink or in other paths.

\cref{tab:synthdata_wide} shows that for reasoning chains about trees of depth 5 with 3 sources, other methods show more significant performance drops while \OURS{} maintains high performance. For wide reasoning trees, other methods don't drop as much, but in trees with greater depth, their error rates increase significantly while \OURS{} remains stable.

\paragraph{Inserting Benign Errors.}
We also insert non-existing rules which don't affect later steps into reasoning chains.
Table~\ref{tab:synthdata_insertion} shows that all methods perform equally well when errors don't cause downstream propagation.

\paragraph{How do different choice of $p$ for base claims and granularity of the entailment model affect the performance?}
We allow using different $p_i$ for the flexibility to not include all the base claims in the premise, and we want to see the impact of different design choices.
\cref{tab:synthdata_long} shows ablations for using $p=1$ vs. $p=0.95$ as well as using granular vs binary entailment models (which use strict \(\{0, 1\}\)).
$p=1$ consistently performs better with a probabilistic entailment model, while a binary entailment model sometimes benefits from $p=0.95$.
Thus we can effectively choose $p=1$, including all base claims, which can greatly reduce the actual computation cost as the effective sample size now only depends on entailment of derived claims.

\subsection{\cert{} Also Improves PRMs}
\label{app:prm}
Process Reward Models (PRMs) can sometimes rival LLMs, and can also provide a non-binary soundness score.
We run additional experiments using a SOTA PRMs, Qwen2.5-Math-PRM-7B, as the base entailment model. The results show that \cert{} can help significantly improve upon PRM on reasoning chains with propagated errors.

The results in \cref{tab:benchmarks-qwenprm} show that, while the specialized PRM is a strong baseline on its in-domain dataset (PRMBench), applying \cert{} significantly improves performance on the abstract \ourdata{} dataset which has many propagated errors. On out-of-domain (non-math) CaptainCook4D, \cert{} achieves on par performance with PRM. This demonstrates \cert{}'s value as a flexible, general-purpose framework that adds robustness, especially on tasks with propagated errors.

\section{AI Assistants in Research}
We used AI assistants (e.g., ChatGPT, Cursor) to support coding frameworks, generate visualizations, and revise writing for clarity and readability.

\begin{table*}[t]
\centering
\scriptsize
\resizebox{\textwidth}{!}{%
\begin{tabular}{%
L{9cm}                  
*{3}{C{0.7cm}} 
*{2}{C{0.9cm}} 
*{2}{C{1.2cm}}          
*{1}{C{0.7cm}} 
*{1}{C{0.9cm}} 
}
\toprule
Claim & \makecell[c]{\textbf{ARES}\\\textbf{(Ours)}} & \makecell[c]{Entail\\-Prev} & \makecell[c]{Entail\\-Base} & \makecell[c]{ReCEval\\-Inter} & \makecell[c]{ReCEval\\-Intra} & \makecell[c]{ROSCOE\\-LI-Source} & \makecell[c]{ROSCOE\\-LI-Self} & \makecell[c]{LLM\\-Judge} & \makecell[c]{\textit{Ground}\\\textit{Truth}} \\
\midrule
sent1: Only after the necessary preceding steps (put-put tomatoes on a serving plate), And if we have all the ingredients, we can then Pour-Pour the egg mixture into the pan. &                        – &                         – &                         – &                           – &                           – &                              – &                            – &                        – &                – \\
\hline
sent2: Only after the necessary preceding steps (Take-Take a tomato), And if we have all the ingredients, we can then Cut-Cut tomato into two pieces. &                        – &                         – &                         – &                           – &                           – &                              – &                            – &                        – &                – \\
\hline
sent3: Only after the necessary preceding steps (Stop-Stop stirring when it's nearly cooked to allow it to set into an omelette), And if we have all the ingredients, we can then Transfer-Transfer omelette to the plate and serve with the tomatoes. &                        – &                         – &                         – &                           – &                           – &                              – &                            – &                        – &                – \\
\hline
sent4: Only after the necessary preceding steps (Chop-Chop 2 tbsp cilantro), And if we have all the ingredients, we can then add-add the chopped cilantro to the bowl. &                        – &                         – &                         – &                           – &                           – &                              – &                            – &                        – &                – \\
\hline
sent5: Only after the necessary preceding steps (START), And if we have all the ingredients, we can then add-1/2 tsp ground black pepper to the bowl. &                        – &                         – &                         – &                           – &                           – &                              – &                            – &                        – &                – \\
\hline
sent6: We have ground black pepper. &                        – &                         – &                         – &                           – &                           – &                              – &                            – &                        – &                – \\
\hline
sent7: We have oil. &                        – &                         – &                         – &                           – &                           – &                              – &                            – &                        – &                – \\
\hline
sent8: Only after the necessary preceding steps (Scoop-Scoop the tomatoes from the pan), And if we have all the ingredients, we can then put-put tomatoes on a serving plate. &                        – &                         – &                         – &                           – &                           – &                              – &                            – &                        – &                – \\
\hline
sent9: Only after the necessary preceding steps (Pour-Pour the egg mixture into the pan), And if we have all the ingredients, we can then stir-stir gently with a wooden spoon so the egg that sets on the base of the pan moves to enable the uncooked egg to flow into the space. &                        – &                         – &                         – &                           – &                           – &                              – &                            – &                        – &                – \\
\hline
sent10: Only after the necessary preceding steps (Transfer-Transfer omelette to the plate and serve with the tomatoes), And if we have all the ingredients, we can then END. &                        – &                         – &                         – &                           – &                           – &                              – &                            – &                        – &                – \\
\hline
sent11: Only after the necessary preceding steps (add-add the chopped cilantro to the bowl, and crack-crack one egg in a bowl, and add-1/2 tsp ground black pepper to the bowl), And if we have all the ingredients, we can then Beat-Beat the contents of the bowl. &                        – &                         – &                         – &                           – &                           – &                              – &                            – &                        – &                – \\
\hline
sent12: Only after the necessary preceding steps (Heat-Heat 1 tbsp oil in a non-stick frying pan), And if we have all the ingredients, we can then cook-cook the tomatoes cut-side down until they start to soften and colour. &                        – &                         – &                         – &                           – &                           – &                              – &                            – &                        – &                – \\
\hline
sent13: Only after the necessary preceding steps (START), And if we have all the ingredients, we can then crack-crack one egg in a bowl. &                        – &                         – &                         – &                           – &                           – &                              – &                            – &                        – &                – \\
\hline
sent14: Only after the necessary preceding steps (cook-cook the tomatoes cut-side down until they start to soften and colour), And if we have all the ingredients, we can then Scoop-Scoop the tomatoes from the pan. &                        – &                         – &                         – &                           – &                           – &                              – &                            – &                        – &                – \\
\hline
sent15: Only after the necessary preceding steps (START), And if we have all the ingredients, we can then Take-Take a tomato. &                        – &                         – &                         – &                           – &                           – &                              – &                            – &                        – &                – \\
\hline
sent16: Only after the necessary preceding steps (Beat-Beat the contents of the bowl, and Cut-Cut tomato into two pieces), And if we have all the ingredients, we can then Heat-Heat 1 tbsp oil in a non-stick frying pan. &                        – &                         – &                         – &                           – &                           – &                              – &                            – &                        – &                – \\
\hline
sent17: We have egg. &                        – &                         – &                         – &                           – &                           – &                              – &                            – &                        – &                – \\
\hline
sent18: Only after the necessary preceding steps (START), And if we have all the ingredients, we can then Chop-Chop 2 tbsp cilantro. &                        – &                         – &                         – &                           – &                           – &                              – &                            – &                        – &                – \\
\hline
sent19: Only after the necessary preceding steps (stir-stir gently with a wooden spoon so the egg that sets on the base of the pan moves to enable the uncooked egg to flow into the space), And if we have all the ingredients, we can then Stop-Stop stirring when it's nearly cooked to allow it to set into an omelette. &                        – &                         – &                         – &                           – &                           – &                              – &                            – &                        – &                – \\
\hline
sent20: We have tomato. &                        – &                         – &                         – &                           – &                           – &                              – &                            – &                        – &                – \\
\hline
sent21: We now START. &                        – &                         – &                         – &                           – &                           – &                              – &                            – &                        – &                – \\
\hline
int1: Because we have completed all previous steps (START), and have all necessary ingredients (cilantro), we can now do the step Chop-Chop 2 tbsp cilantro. And now we have completed this step Chop-Chop 2 tbsp cilantro. &       \textbf{0.35}\redx &        \textbf{0.00}\redx &        \textbf{0.00}\redx &          \textbf{0.00}\redx &             1.00\greencheck &             \textbf{0.00}\redx &              1.00\greencheck &          1.00\greencheck &            \redx \\
\hline
int2: Because we have completed all previous steps (START), and have all necessary ingredients (egg), we can now do the step crack-crack one egg in a bowl. And now we have completed this step crack-crack one egg in a bowl. & \textbf{0.85}\greencheck &  \textbf{1.00}\greencheck &  \textbf{1.00}\greencheck &                   0.00\redx &    \textbf{1.00}\greencheck &                      0.00\redx &                    0.00\redx & \textbf{1.00}\greencheck &      \greencheck \\
\hline
int3: Because we have completed all previous steps (START), and have all necessary ingredients (tomato), we can now do the step Take-Take a tomato. And now we have completed this step Take-Take a tomato. & \textbf{0.98}\greencheck &  \textbf{1.00}\greencheck &  \textbf{1.00}\greencheck &                   0.00\redx &    \textbf{1.00}\greencheck &                      0.00\redx &                    0.00\redx & \textbf{1.00}\greencheck &      \greencheck \\
\hline
int4: Because we have completed all previous steps (START), and have all necessary ingredients (ground black pepper), we can now do the step add-1/2 tsp ground black pepper to the bowl. And now we have completed this step add-1/2 tsp ground black pepper to the bowl. & \textbf{0.80}\greencheck &  \textbf{1.00}\greencheck &  \textbf{1.00}\greencheck &                   0.00\redx &    \textbf{1.00}\greencheck &                      0.00\redx &     \textbf{1.00}\greencheck & \textbf{1.00}\greencheck &      \greencheck \\
\hline
int5: Because we have completed all previous steps (Chop-Chop 2 tbsp cilantro), and have all necessary ingredients (cilantro), we can now do the step add-add the chopped cilantro to the bowl. And now we have completed this step add-add the chopped cilantro to the bowl. &       \textbf{0.00}\redx &        \textbf{0.00}\redx &        \textbf{0.00}\redx &          \textbf{0.00}\redx &             1.00\greencheck &             \textbf{0.00}\redx &           \textbf{0.00}\redx &          1.00\greencheck &            \redx \\
\hline
int6: Because we have completed all previous steps (Take-Take a tomato), and have all necessary ingredients (tomato), we can now do the step Cut-Cut tomato into two pieces. And now we have completed this step Cut-Cut tomato into two pieces. & \textbf{0.96}\greencheck &  \textbf{1.00}\greencheck &  \textbf{1.00}\greencheck &                   0.00\redx &    \textbf{1.00}\greencheck &                      0.00\redx &                    0.00\redx & \textbf{1.00}\greencheck &      \greencheck \\
\hline
int7: Because we have completed all previous steps (add-add the chopped cilantro to the bowl, and crack-crack one egg in a bowl, and add-1/2 tsp ground black pepper to the bowl),  we can now do the step Beat-Beat the contents of the bowl. And now we have completed this step Beat-Beat the contents of the bowl. &       \textbf{0.01}\redx &        \textbf{0.00}\redx &           1.00\greencheck &          \textbf{0.00}\redx &             1.00\greencheck &             \textbf{0.00}\redx &           \textbf{0.00}\redx &          1.00\greencheck &            \redx \\
\hline
int8: Because we have completed all previous steps (Beat-Beat the contents of the bowl, and Cut-Cut tomato into two pieces), and have all necessary ingredients (oil), we can now do the step Heat-Heat 1 tbsp oil in a non-stick frying pan. And now we have completed this step Heat-Heat 1 tbsp oil in a non-stick frying pan. &       \textbf{0.00}\redx &        \textbf{0.00}\redx &        \textbf{0.00}\redx &          \textbf{0.00}\redx &             1.00\greencheck &             \textbf{0.00}\redx &           \textbf{0.00}\redx &          1.00\greencheck &            \redx \\
\hline
int9: Because we have completed all previous steps (Heat-Heat 1 tbsp oil in a non-stick frying pan), and have all necessary ingredients (tomatoes), we can now do the step cook-cook the tomatoes cut-side down until they start to soften and colour. And now we have completed this step cook-cook the tomatoes cut-side down until they start to soften and colour. &       \textbf{0.01}\redx &           1.00\greencheck &           1.00\greencheck &          \textbf{0.00}\redx &             1.00\greencheck &             \textbf{0.00}\redx &           \textbf{0.00}\redx &          1.00\greencheck &            \redx \\
\hline
int10: Because we have completed all previous steps (cook-cook the tomatoes cut-side down until they start to soften and colour),  we can now do the step Scoop-Scoop the tomatoes from the pan. And now we have completed this step Scoop-Scoop the tomatoes from the pan. &       \textbf{0.21}\redx &           1.00\greencheck &           1.00\greencheck &          \textbf{0.00}\redx &             1.00\greencheck &             \textbf{0.00}\redx &           \textbf{0.00}\redx &          1.00\greencheck &            \redx \\
\hline
int11: Because we have completed all previous steps (Scoop-Scoop the tomatoes from the pan),  we can now do the step put-put tomatoes on a serving plate. And now we have completed this step put-put tomatoes on a serving plate. &       \textbf{0.18}\redx &           1.00\greencheck &           1.00\greencheck &          \textbf{0.00}\redx &          \textbf{0.00}\redx &             \textbf{0.00}\redx &           \textbf{0.00}\redx &          1.00\greencheck &            \redx \\
\hline
int12: Because we have completed all previous steps (put-put tomatoes on a serving plate),  we can now do the step Pour-Pour the egg mixture into the pan. And now we have completed this step Pour-Pour the egg mixture into the pan. &       \textbf{0.18}\redx &           1.00\greencheck &        \textbf{0.00}\redx &          \textbf{0.00}\redx &          \textbf{0.00}\redx &             \textbf{0.00}\redx &           \textbf{0.00}\redx &          1.00\greencheck &            \redx \\
\hline
int13: Because we have completed all previous steps (Pour-Pour the egg mixture into the pan),  we can now do the step stir-stir gently with a wooden spoon so the egg that sets on the base of the pan moves to enable the uncooked egg to flow into the space. And now we have completed this step stir-stir gently with a wooden spoon so the egg that sets on the base of the pan moves to enable the uncooked egg to flow into the space. &       \textbf{0.19}\redx &           1.00\greencheck &        \textbf{0.00}\redx &          \textbf{0.00}\redx &          \textbf{0.00}\redx &             \textbf{0.00}\redx &           \textbf{0.00}\redx &          1.00\greencheck &            \redx \\
\hline
int14: Because we have completed all previous steps (stir-stir gently with a wooden spoon so the egg that sets on the base of the pan moves to enable the uncooked egg to flow into the space),  we can now do the step Stop-Stop stirring when it's nearly cooked to allow it to set into an omelette. And now we have completed this step Stop-Stop stirring when it's nearly cooked to allow it to set into an omelette. &       \textbf{0.19}\redx &           1.00\greencheck &        \textbf{0.00}\redx &          \textbf{0.00}\redx &             1.00\greencheck &             \textbf{0.00}\redx &           \textbf{0.00}\redx &          1.00\greencheck &            \redx \\
\hline
int15: Because we have completed all previous steps (Stop-Stop stirring when it's nearly cooked to allow it to set into an omelette),  we can now do the step Transfer-Transfer omelette to the plate and serve with the tomatoes. And now we have completed this step Transfer-Transfer omelette to the plate and serve with the tomatoes. &       \textbf{0.00}\redx &           1.00\greencheck &        \textbf{0.00}\redx &          \textbf{0.00}\redx &             1.00\greencheck &             \textbf{0.00}\redx &           \textbf{0.00}\redx &          1.00\greencheck &            \redx \\
\hline
int16: Because we have completed all previous steps (Transfer-Transfer omelette to the plate and serve with the tomatoes),  we can now do the step END. And now we have completed this step END. &       \textbf{0.00}\redx &           1.00\greencheck &        \textbf{0.00}\redx &          \textbf{0.00}\redx &             1.00\greencheck &             \textbf{0.00}\redx &           \textbf{0.00}\redx &          1.00\greencheck &            \redx \\
\bottomrule
\end{tabular}
}
\caption{\textbf{(\recipes{} Example)} Only \OURS{} is able to correctly judge all steps for soundness. Checks~\greencheck{} indicate that a method classifies the step as sound after thresholding, and crosses~\redx{} indicate that the method judges that step to be erroneous. \textbf{Bold}: Correctly judged soundness.}
\label{tab:recipe-example}
\end{table*}
\begin{table*}[t]
  \centering
  \small
\begin{tabular}{lccc}
\toprule
Dataset / Method & Recall & Precision & F1 \\
\midrule
\multicolumn{4}{l}{\textbf{\ourdata{}-s3d3}} \\
\cert{}-1 & \textbf{0.921$\pm$ 0.102} & \textbf{0.980$\pm$ 0.018} & \textbf{0.941$\pm$ 0.074} \\
\cert{}-0.95 & \underline{0.904$\pm$ 0.110} & \underline{0.975$\pm$ 0.027} & \underline{0.927$\pm$ 0.081} \\
\entailprev{} & 0.821$\pm$ 0.046 & 0.951$\pm$ 0.032 & 0.863$\pm$ 0.039 \\
\entailraw{} & 0.859$\pm$ 0.122 & 0.866$\pm$ 0.142 & 0.837$\pm$ 0.134 \\
ROSCOE-LI-Self & 0.500$\pm$ 0.000 & 0.115$\pm$ 0.060 & 0.181$\pm$ 0.078 \\
ROSCOE-LI-Source & 0.623$\pm$ 0.101 & 0.593$\pm$ 0.087 & 0.497$\pm$ 0.161 \\
ReCEval-Intra & 0.500$\pm$ 0.000 & 0.115$\pm$ 0.060 & 0.181$\pm$ 0.078 \\
ReCEval-Inter & 0.585$\pm$ 0.081 & 0.562$\pm$ 0.061 & 0.449$\pm$ 0.115 \\
\llmjudge{} & 0.833$\pm$ 0.051 & 0.957$\pm$ 0.022 & 0.875$\pm$ 0.035 \\
\midrule
\multicolumn{4}{l}{\textbf{\ourdata{}-s3d5}} \\
\cert{}-0.95 & \textbf{0.867$\pm$ 0.171} & \textbf{0.971$\pm$ 0.037} & \textbf{0.887$\pm$ 0.146} \\
\entailprev{} & 0.718$\pm$ 0.090 & 0.936$\pm$ 0.045 & 0.761$\pm$ 0.097 \\
\entailraw{} & 0.659$\pm$ 0.061 & 0.618$\pm$ 0.076 & 0.610$\pm$ 0.091 \\
ROSCOE-LI-Self & 0.497$\pm$ 0.044 & 0.500$\pm$ 0.242 & 0.460$\pm$ 0.074 \\
ROSCOE-LI-Source & 0.513$\pm$ 0.117 & 0.514$\pm$ 0.077 & 0.340$\pm$ 0.081 \\
ReCEval-Intra & 0.500$\pm$ 0.000 & 0.100$\pm$ 0.054 & 0.161$\pm$ 0.074 \\
ReCEval-Inter & 0.550$\pm$ 0.070 & 0.539$\pm$ 0.050 & 0.356$\pm$ 0.083 \\
\llmjudge{} & \underline{0.774$\pm$ 0.178} & \underline{0.942$\pm$ 0.057} & \underline{0.796$\pm$ 0.169} \\
\midrule
\multicolumn{4}{l}{\textbf{\ourdata{}-s5d3}} \\
\cert{}-1 & \textbf{0.875$\pm$ 0.217} & \underline{0.889$\pm$ 0.232} & \textbf{0.880$\pm$ 0.223} \\
\cert{}-0.95 & \underline{0.867$\pm$ 0.217} & \textbf{0.889$\pm$ 0.232} & \underline{0.875$\pm$ 0.222} \\
\entailprev{} & 0.767$\pm$ 0.181 & 0.873$\pm$ 0.223 & 0.799$\pm$ 0.191 \\
\entailraw{} & 0.824$\pm$ 0.205 & 0.700$\pm$ 0.149 & 0.729$\pm$ 0.167 \\
ROSCOE-LI-Self & 0.500$\pm$ 0.000 & 0.055$\pm$ 0.033 & 0.097$\pm$ 0.052 \\
ROSCOE-LI-Source & 0.650$\pm$ 0.054 & 0.560$\pm$ 0.031 & 0.380$\pm$ 0.073 \\
ReCEval-Intra & 0.500$\pm$ 0.000 & 0.055$\pm$ 0.033 & 0.097$\pm$ 0.052 \\
ReCEval-Inter & 0.594$\pm$ 0.095 & 0.539$\pm$ 0.043 & 0.357$\pm$ 0.063 \\
\llmjudge{} & 0.742$\pm$ 0.192 & 0.868$\pm$ 0.222 & 0.770$\pm$ 0.201 \\
\midrule
\multicolumn{4}{l}{\textbf{\ourdata{}-s5d5}} \\
\cert{}-1 & \textbf{0.900$\pm$ 0.163} & \textbf{0.990$\pm$ 0.017} & \textbf{0.920$\pm$ 0.139} \\
\cert{}-0.95 & \textbf{0.900$\pm$ 0.163} & \textbf{0.990$\pm$ 0.017} & \textbf{0.920$\pm$ 0.139} \\
\entailprev{} & 0.723$\pm$ 0.096 & \underline{0.969$\pm$ 0.018} & 0.783$\pm$ 0.095 \\
\entailraw{} & 0.692$\pm$ 0.141 & 0.597$\pm$ 0.067 & 0.610$\pm$ 0.083 \\
ROSCOE-LI-Self & 0.481$\pm$ 0.020 & 0.446$\pm$ 0.018 & 0.462$\pm$ 0.010 \\
ROSCOE-LI-Source & 0.578$\pm$ 0.063 & 0.533$\pm$ 0.027 & 0.321$\pm$ 0.055 \\
ReCEval-Intra & 0.500$\pm$ 0.000 & 0.053$\pm$ 0.019 & 0.094$\pm$ 0.031 \\
ReCEval-Inter & 0.584$\pm$ 0.097 & 0.534$\pm$ 0.059 & 0.310$\pm$ 0.084 \\
\llmjudge{} & \underline{0.847$\pm$ 0.140} & 0.951$\pm$ 0.082 & \underline{0.881$\pm$ 0.111} \\
\midrule
\bottomrule
\end{tabular}
  \caption{\textbf{GPT-4o-mini} \textbf{(\ourdata{})} \cert{} differs from other methods in deeper trees instead of wider trees. s3d5 means trees with 3 sources and depth of 5.}
  \label{tab:synthdata_wide}
\end{table*}

\begin{table*}[t]
  \centering
  \small
\begin{tabular}{lccc}
\toprule
Dataset / Method & Recall & Precision & F1 \\
\midrule
\multicolumn{4}{l}{\textbf{\ourdata{}-v5i1}} \\
\cert{}-1 & 0.985$\pm$ 0.014 & 0.950$\pm$ 0.046 & 0.965$\pm$ 0.032 \\
\cert{}-0.95 & 0.990$\pm$ 0.022 & \underline{0.998$\pm$ 0.005} & \underline{0.994$\pm$ 0.015} \\
\entailprev{} & \underline{0.992$\pm$ 0.011} & 0.974$\pm$ 0.038 & 0.982$\pm$ 0.026 \\
\entailraw{} & 0.900$\pm$ 0.027 & 0.788$\pm$ 0.030 & 0.813$\pm$ 0.038 \\
ROSCOE-LI-Self & 0.975$\pm$ 0.009 & 0.918$\pm$ 0.025 & 0.942$\pm$ 0.019 \\
ROSCOE-LI-Source & 0.690$\pm$ 0.062 & 0.626$\pm$ 0.038 & 0.545$\pm$ 0.058 \\
ReCEval-Intra & 0.500$\pm$ 0.000 & 0.100$\pm$ 0.000 & 0.167$\pm$ 0.000 \\
ReCEval-Inter & 0.755$\pm$ 0.047 & 0.671$\pm$ 0.021 & 0.590$\pm$ 0.066 \\
\llmjudge{} & \textbf{1.000$\pm$ 0.000} & \textbf{1.000$\pm$ 0.000} & \textbf{1.000$\pm$ 0.000} \\
\midrule
\multicolumn{4}{l}{\textbf{\ourdata{}-v5i2}} \\
\cert{}-1 & \textbf{1.000$\pm$ 0.000} & \textbf{1.000$\pm$ 0.000} & \textbf{1.000$\pm$ 0.000} \\
\cert{}-0.95 & \underline{0.995$\pm$ 0.011} & \underline{0.998$\pm$ 0.005} & \underline{0.996$\pm$ 0.008} \\
\entailprev{} & 0.990$\pm$ 0.010 & 0.981$\pm$ 0.019 & 0.985$\pm$ 0.015 \\
\entailraw{} & 0.863$\pm$ 0.009 & 0.823$\pm$ 0.007 & 0.815$\pm$ 0.013 \\
ROSCOE-LI-Self & 0.965$\pm$ 0.030 & 0.951$\pm$ 0.036 & 0.956$\pm$ 0.033 \\
ROSCOE-LI-Source & 0.635$\pm$ 0.054 & 0.642$\pm$ 0.058 & 0.555$\pm$ 0.057 \\
ReCEval-Intra & 0.500$\pm$ 0.000 & 0.167$\pm$ 0.000 & 0.250$\pm$ 0.000 \\
ReCEval-Inter & 0.695$\pm$ 0.029 & 0.721$\pm$ 0.020 & 0.594$\pm$ 0.038 \\
\llmjudge{} & 0.978$\pm$ 0.016 & 0.960$\pm$ 0.028 & 0.967$\pm$ 0.024 \\
\midrule
\multicolumn{4}{l}{\textbf{\ourdata{}-v5i5}} \\
\cert{}-1 & \underline{0.988$\pm$ 0.028} & \underline{0.991$\pm$ 0.020} & 0.989$\pm$ 0.026 \\
\cert{}-0.95 & \textbf{0.998$\pm$ 0.004} & \textbf{0.998$\pm$ 0.005} & \textbf{0.998$\pm$ 0.005} \\
\entailprev{} & \underline{0.988$\pm$ 0.013} & 0.990$\pm$ 0.010 & \underline{0.989$\pm$ 0.011} \\
\entailraw{} & 0.930$\pm$ 0.019 & 0.950$\pm$ 0.012 & 0.936$\pm$ 0.018 \\
ROSCOE-LI-Self & 0.938$\pm$ 0.012 & 0.955$\pm$ 0.008 & 0.943$\pm$ 0.012 \\
ROSCOE-LI-Source & 0.661$\pm$ 0.006 & 0.736$\pm$ 0.033 & 0.649$\pm$ 0.011 \\
ReCEval-Intra & 0.500$\pm$ 0.000 & 0.278$\pm$ 0.000 & 0.357$\pm$ 0.000 \\
ReCEval-Inter & 0.665$\pm$ 0.024 & 0.826$\pm$ 0.008 & 0.642$\pm$ 0.034 \\
\llmjudge{} & 0.982$\pm$ 0.017 & 0.983$\pm$ 0.017 & 0.982$\pm$ 0.017 \\
\midrule
\bottomrule
\end{tabular}
\caption{\textbf{GPT-4o-mini} \textbf{(\ourdata{})} \cert{} does not differ much from other methods in inserted errors that do not affect downstream reasoning. v5i2 means 5 valid claims and 2 inserted claims.}
  \label{tab:synthdata_insertion}
\end{table*}

\begin{table*}[t]
  \centering
  \small
\begin{tabular}{lccc}
\toprule
Dataset / Method & Recall & Precision & F1 \\
\midrule
\multicolumn{4}{l}{\textbf{\ourdata{}-5}} \\
\cert{}-1 & 0.881 & 0.900 & 0.873 \\
\cert{}-0.95 & 0.861 & 0.889 & 0.854 \\
\cert{}-bin-1 & \underline{0.898} & \underline{0.913} & \underline{0.891} \\
\cert{}-bin-0.95 & \textbf{0.909} & \textbf{0.919} & \textbf{0.902} \\
\entailprev{} & 0.704 & 0.813 & 0.673 \\
\entailraw{} & 0.830 & 0.832 & 0.824 \\
ROSCOE-LI-Self & 0.499 & 0.500 & 0.351 \\
ROSCOE-LI-Source & 0.647 & 0.650 & 0.640 \\
ReCEval-Intra & 0.500 & 0.250 & 0.332 \\
ReCEval-Inter & 0.645 & 0.648 & 0.638 \\
\llmjudge{} & 0.811 & 0.864 & 0.803 \\
\midrule
\multicolumn{4}{l}{\textbf{\ourdata{}-10}} \\
\cert{}-1 & 0.937 & 0.943 & 0.936 \\
\cert{}-0.95 & 0.931 & 0.936 & 0.931 \\
\cert{}-bin-1 & \textbf{0.960} & \textbf{0.965} & \textbf{0.962} \\
\cert{}-bin-0.95 & \underline{0.947} & \underline{0.951} & \underline{0.948} \\
\entailprev{} & 0.608 & 0.783 & 0.538 \\
\entailraw{} & 0.626 & 0.636 & 0.616 \\
ROSCOE-LI-Self & 0.524 & 0.589 & 0.420 \\
ROSCOE-LI-Source & 0.544 & 0.548 & 0.533 \\
ReCEval-Intra & 0.500 & 0.247 & 0.330 \\
ReCEval-Inter & 0.566 & 0.573 & 0.555 \\
\llmjudge{} & 0.767 & 0.839 & 0.750 \\
\midrule
\multicolumn{4}{l}{\textbf{\ourdata{}-20}} \\
\cert{}-1 & \textbf{0.979} & \textbf{0.979} & \textbf{0.978} \\
\cert{}-0.95 & \underline{0.971} & \underline{0.971} & \underline{0.971} \\
\cert{}-bin-1 & 0.964 & 0.966 & 0.963 \\
\cert{}-bin-0.95 & 0.968 & 0.970 & 0.968 \\
\entailprev{} & 0.551 & 0.760 & 0.440 \\
\entailraw{} & 0.533 & 0.537 & 0.522 \\
ROSCOE-LI-Self & 0.521 & 0.580 & 0.414 \\
ROSCOE-LI-Source & 0.508 & 0.509 & 0.480 \\
ReCEval-Intra & 0.500 & 0.248 & 0.331 \\
ReCEval-Inter & 0.513 & 0.516 & 0.482 \\
\llmjudge{} & 0.640 & 0.788 & 0.586 \\
\midrule
\multicolumn{4}{l}{\textbf{\ourdata{}-30}} \\
\cert{}-1 & \textbf{0.973} & \underline{0.972} & \textbf{0.971} \\
\cert{}-0.95 & 0.931 & 0.934 & 0.929 \\
\cert{}-bin-1 & \underline{0.967} & \textbf{0.973} & \underline{0.969} \\
\cert{}-bin-0.95 & 0.957 & 0.960 & 0.956 \\
\entailprev{} & 0.530 & 0.731 & 0.387 \\
\entailraw{} & 0.531 & 0.539 & 0.499 \\
ROSCOE-LI-Self & 0.543 & 0.595 & 0.460 \\
ROSCOE-LI-Source & 0.498 & 0.498 & 0.461 \\
ReCEval-Intra & 0.500 & 0.262 & 0.343 \\
ReCEval-Inter & 0.506 & 0.509 & 0.464 \\
\llmjudge{} & 0.581 & 0.757 & 0.482 \\
\midrule
\multicolumn{4}{l}{\textbf{\ourdata{}-50}} \\
\cert{}-1 & \textbf{0.895} & \underline{0.899} & \textbf{0.890} \\
\cert{}-0.95 & 0.871 & 0.871 & 0.867 \\
\cert{}-bin-1 & 0.887 & \textbf{0.904} & 0.886 \\
\cert{}-bin-0.95 & \underline{0.892} & 0.892 & \underline{0.888} \\
\entailprev{} & 0.512 & 0.601 & 0.340 \\
\entailraw{} & 0.507 & 0.508 & 0.486 \\
ROSCOE-LI-Self & 0.555 & 0.581 & 0.504 \\
ROSCOE-LI-Source & 0.505 & 0.509 & 0.442 \\
ReCEval-Intra & 0.500 & 0.262 & 0.343 \\
ReCEval-Inter & 0.498 & 0.496 & 0.428 \\
\llmjudge{} & 0.529 & 0.714 & 0.385 \\
\midrule
\bottomrule
\end{tabular}
\caption{\textbf{GPT-4o-mini} \textbf{(\ourdata{})} \cert{} consistently identifies errors in long reasoning chains while other methods gradually fail.}
  \label{tab:synthdata_long}
\end{table*}

\end{document}